\def\year{2022}\relax
\documentclass[letterpaper]{article} 
\usepackage{aaai22}  
\usepackage{times}  
\usepackage{helvet}  
\usepackage{courier}  
\usepackage[hyphens]{url}  
\usepackage{graphicx} 
\usepackage{natbib}  
\usepackage{caption} 
\DeclareCaptionStyle{ruled}{labelfont=normalfont,labelsep=colon,strut=off} 
\usepackage{algorithm}
\usepackage{algorithmic}
\usepackage{wrapfig}
\usepackage{color}
\usepackage{bm}
\usepackage{booktabs}
\usepackage{subfigure}
\usepackage{extarrows}
\usepackage{amsmath}
\usepackage{ntheorem}
\usepackage{amsfonts}
\newtheorem{example}{Example}
\newtheorem{theorem}{Theorem}
\newtheorem{definition}{Definition}

\newtheorem{assumption}{Assumption}
\newtheorem*{proof}{Proof}
\usepackage{newfloat}
\usepackage{listings}
\floatstyle{ruled}
\newfloat{listing}{tb}{lst}{}
\floatname{listing}{Listing}
\title{ER: Equivariance Regularizer for \ 
Knowledge Graph Completion}
\author{
Zongsheng Cao \textsuperscript{\rm 1,2},
   Qianqian Xu \textsuperscript{\rm 3,*},
   Zhiyong Yang \textsuperscript{\rm 4},
   Qingming Huang \textsuperscript{\rm 3,4,5,6,*}
}
\affiliations{
\textsuperscript{\rm 1} State Key Laboratory of Information Security, Institute of Information Engineering, CAS, Beijing, China \\
   \textsuperscript{\rm 2} School of Cyber Security, University of Chinese Academy of Sciences, Beijing, China \\
   \textsuperscript{\rm 3} Key Laboratory of Intelligent Information Processing, Institute of Computing Technology, CAS, Beijing, China \\
   \textsuperscript{\rm 4} School of Computer Science and Technology, University of Chinese Academy of Sciences, Beijing, China \\
   \textsuperscript{\rm 5} Key Laboratory of Big Data Mining and Knowledge Management, Chinese Academy of Sciences, Beijing, China \\
   \textsuperscript{\rm 6} Peng Cheng Laboratory, Shenzhen, China \\
caozongsheng@iie.ac.cn, xuqianqian@ict.ac.cn, yangzhiyong21@ucas.ac.cn, qmhuang@ucas.ac.cn
}
\usepackage{bibentry}

\begin{document}

\maketitle

\begin{abstract}
 Tensor factorization and distanced based models  play  important roles in knowledge graph
completion (KGC). However,  the relational matrices in KGC methods often induce a high model complexity, bearing a high risk of overfitting. As a remedy, researchers propose a variety of different regularizers such as the tensor nuclear norm regularizer. 
Our motivation  is based on the observation
that the previous work only focuses on the ``size" of the parametric space, while leaving the implicit semantic information widely untouched.
To address this issue, we propose a new
 regularizer, namely, 
Equivariance Regularizer  (\textbf{ER}), 
    which can suppress overfitting by leveraging the implicit semantic information. Specifically, ER can enhance the generalization ability of the model by employing the semantic equivariance between the head and tail entities. Moreover, it is
a generic solution for both
 distance based models and tensor factorization based models.  The experimental results indicate a clear and substantial improvement over the state-of-the-art relation prediction methods.
\end{abstract}

\section{Introduction}
 Knowledge Graph (KG) represents a collection of interlinked descriptions of entities, namely, real-world objects, events, and abstract concepts.
  Knowledge graphs are applied for a wide spectrum of applications ranging from question answering \cite{DBLP:journals/corr/abs-2012-15484}, natural
language processing \cite{DBLP:conf/acl/ZhangHLJSL19},  computer
vision \cite{DBLP:conf/cvpr/MarinoSG17} and recommendation systems \cite{DBLP:conf/cikm/WangZWZLXG18}. However, the knowledge graph is usually incomplete with  missing relations between entities. 
 To  predict missing links between entities based on known
links effectively, 
a key step, known as Knowledge Graph Completion (KGC) has attracted  increasing attention from the community.

There are mainly two important branches of KGC models: distance based (DB) models and tensor factorization based (TFB) models.
As for the former, they  use the Minkowski distance to
 calculate  the similarity between vectors of entities,
 and they can achieve reasonable performance relying on the geometric significance.
As for the latter, they 
treat knowledge graph completion as a tensor
completion problem so that these models are highly expressive in theory. 
A noteworthy fact is that the relation-specific matrices in DB and TFB models contain lots of
parameters, which can embrace rich potential relations. However,  the performance of such models  usually suffers from the overfitting problem since the relational matrices induce a high
model complexity. For example,  \cite{DBLP:journals/tkde/WangMWG17} has shown  models such as RESCAL \cite{2011RESCAL} are overfitting.

To address this challenge, researchers propose various regularizers. The
squared Frobenius norm regularizer is a popular choice  \cite{DBLP:conf/nips/SalakhutdinovS10,yang2014embedding,trouillon2016complex} for its convenience in applications.
However, since the regularizer  only imposes constraints to limit the parameter space (e.g., entities  and  relations),  it cannot effectively improve the performance of some models with more implicit constraints  \cite{conf/iclr/RuffinelliBG20}  (e.g., RESCAL). 
Futhermore,  
\cite{lacroix2018canonical} proposes a tensor nuclear p-norm regularizer,  which encourages to use the matrix trace norm in the matrix completion problem \cite{DBLP:conf/nips/SrebroRJ04,DBLP:journals/focm/CandesR09}. 
 Unfortunately, it is only suitable for
canonical polyadic (CP) decomposition \cite{Hitchcock1927The} based models such as ComplEx \cite{trouillon2016complex}.
For models with more complex mechanisms such as RESCAL \cite{2011RESCAL}, the tensor nuclear p-norm regularizer cannot utilize the potential semantic information so that it is weak in suppressing overfitting.
Recently, \cite{DBLP:conf/nips/ZhangC020} proposes a regularizer called DURA  for TFB models. It essentially imposes constraints on linked entities, which  is deficient in exploring  latent semantic relations.
In summary, most of the aforementioned methods overlook the latent semantic relations and could only improve some specific models for KGC.
How to find an efficient regularizer embracing both DB and TFB models remains wide open.

\begin{figure}
	\centering
\includegraphics[width=2.8in]{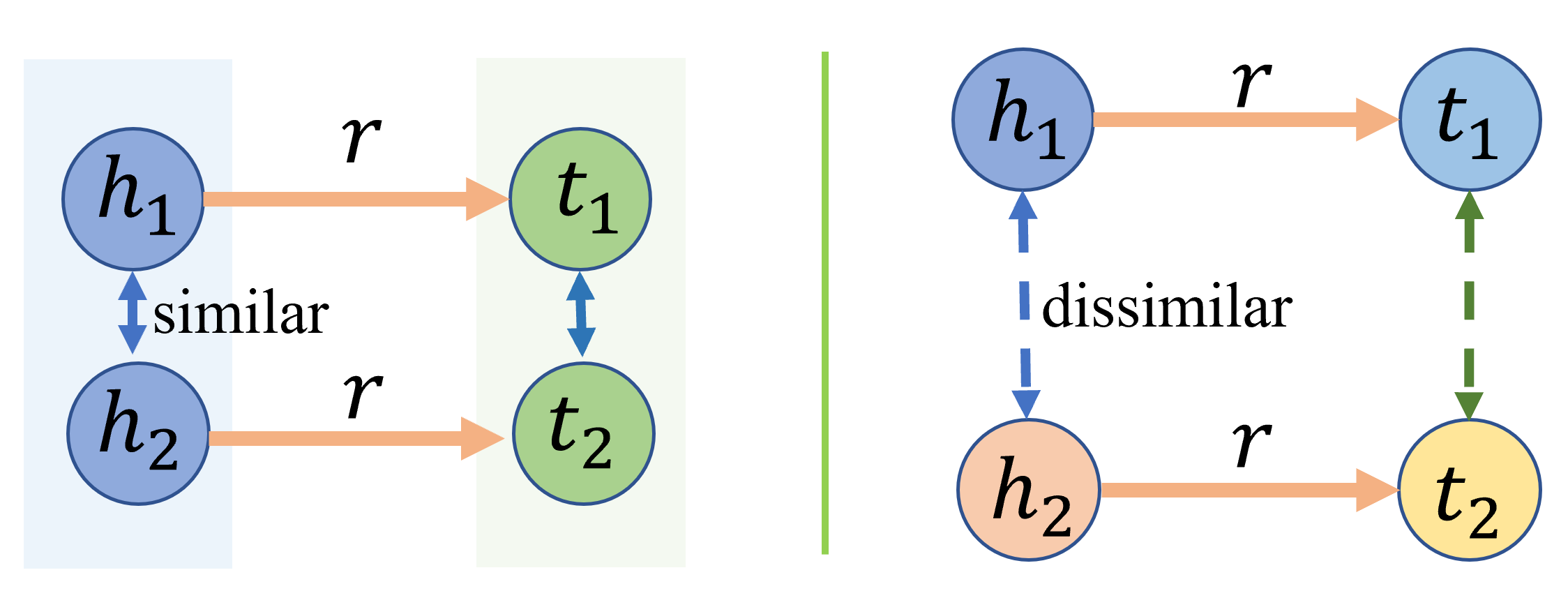}
\caption{An illustration for the equivariance of proximity (left part in the figure) and dissimilarity (right part in the figure). $r$ denotes $relation$, $h_i$ denotes $head_i$ and $t_i$ denotes $tail_i$ ($i=1,2$).}\label{Equivariance}
\end{figure}
In this paper, we propose a novel regularizer for both tensor factorization and distance based KGC models called
Equivariance Regularizer  (\textbf{EA}). As mentioned above, our motivation  is based on the observation
that previous work only focuses on limiting the parameter space, overlooking the potential  semantic relation between entities. 
 For example, given triples $(head_1,relation,tail_1)$ and $(head_2,relation,tail_2)$ shown in FIgure \ref{Equivariance},   if the semantics of $head_1$ and $head_2$ are similar,  the semantics of $tail_1$ and $tail_2$ respectively linked by same $relation$ should be also similar generally, and vice versa.
 Motivated by this,
  different from previous work, we attempt to  utilize the latent semantic information between entities to 
 suppress overfitting,   
which can be  realized  on top of the equivariance of proximity and dissimilarity jointly. In light of this, ER can be applied to a variety of DB and TFB models, including  RESCAL,  ComplEx and RotatE \cite{sun2019rotate:}. Experiments show that ER can
yield consistent and significant improvements on datasets for the knowledge graph completion task.

 Our contributions can be summarized as follows. (1) To the best
of our knowledge, our regularizer is the first to focus on exploring potential semantic relations based on the equivariance of proximity and dissimilarity. 
  (2) We provide a reformulation of our regularization, which shows the relationship between our work and traditional regularizers. Moreover, ER is a flexible regularizer that can be applied to both DB and TFB models.  (3) We evaluate our model to address the challenge of relation prediction tasks for a wide variety of real-world datasets,  
and show that ER can produce sharp improvement on benchmarks.


\section{Related Work}\label{Related Work}
Knowledge graph embedding is  an important research direction in representation
learning.  
Therefore, a number of approaches have been developed for embedding
KGs. We roughly divide previous KGC models into  distance based models and tensor factorization based models.

Distance based (DB) models can model relations and entities by  matrix or vector  maps for a triple $(head, relation,tail)$. Specifically, Minkowski distance is applied to calculate the similarity between vectors of entities. Translational methods proposed first by TransE \cite{bordes2013translating} are widely used embedding
methods, which interpret relation vectors as translations in vector space, i.e., $head + relation \approx tail$. A number of models  aiming to improve TransE are proposed subsequently, such as TransH \cite{wang2014knowledge}, TransR \cite{lin2015learning} and TransD \cite{ji2015knowledge}. They use the score function like the formulation of $s(\bm{x}_i,\bm{R}_j,\bm{x}_k)= -||f(\bm{x}_i,\bm{R}_j,\bm{x}_k)||$,
where $f$ is a model-specific function. 
Recently, there are some  effective DB models  proposed such as RotatE \cite{sun2019rotate:}, \textsc {RotH} \cite{chami2020low} and OTE \cite{tang-etal-2020-orthogonal}, which have made great progress  in the sense of geometric interpretation.
However, some DB models (e.g., TransR \cite{lin2015learning}) have shortcomings in the expressive ability of the model since the DB models perform KGC based solely on observed facts rather than utilizing latent semantic information.

Tensor factorization based (TFB) models formulate the KGC task as a third-order binary tensor completion
problem.  Denote 
the head entity (tail entity) 
as $\bm{x}_i(\bm{x}_k)$, and denote the matrix representing relation as $\bm{R}_j$. Suppose $\bm{Re}(\cdot)$ denotes the real
part, and $\overline{\cdot}$ denotes the conjugate of a complex matrix. Then TFB models factorize the third-order tensor as 
$\bm{Re}(\overline{\bm{H}}\bm{R}_j\bm{T}^{\top})$. Specifically, the score functions are defined
as $s(\bm{x}_i,\bm{R}_j,\bm{x}_k) = \bm{Re}(\overline{\bm{x}}_i \bm{R}_j \bm{x}^{\top}_k)$. RESCAL \cite{2011RESCAL} proposes the three-way rank-r factorization over each relational slice of knowledge graph tensor. In this method a large amount of parameters are contained in the relation-specifific matrices and it bring a risk of overfifitting thereby.
 By restricting relation matrix  to be diagonal
for multi-relational representation learning, DistMult \cite{yang2014embedding}
proposes a simplified bilinear formulation but it cannot handle asymmetric relations.
Going a step further, ComplEx \cite{trouillon2016complex} 
proposes  to embed entities  using complex vectors, which can capture symmetric and antisymmetric
relations.
Moreover, TuckER \cite{balavzevic2019tucker}  learns embeddings by outputting a core tensor and embedding vectors
of entities and relations. But due to the existence of overfitting \cite{2011RESCAL}, their performance lags behind the DB model. 
There are also some neural network models such as \cite{DBLP:journals/corr/abs-2101-01669},
\cite{AS_GCN} \cite{DBLP:conf/www/JinHL021} and \cite{nathani2019learning}, but they also have the risk of overfitting due to the huge number of parameters.

 To  tackle the overfitting problem of KGC models and enhance the performance of these models,
  researchers propose various regularizers. In the previous work,  squared Frobenius norm ( L2 norm) regularizer  is usually applied in TFB models \cite{2011RESCAL,yang2014embedding,trouillon2016complex}. However, TFB models can not achieve comparable performance as distance
based models \cite{sun2019rotate:,2019LearningHierarchyAware}.
Then \cite{lacroix2018canonical} proposes to use the tensor nuclear 3-norm \cite{2014Nuclear} (N3) as a regularizer. However, it is designed for some specific models such as  ComplEx, and it is
not appropriate for other general models such as RESCAL.
Moreover, \cite{DBLP:conf/pkdd/MinerviniCMNV17} uses a set of model-dependent soft constraints imposed on the predicate embeddings to model the equivalence and inversion patterns. There are also some  methods such as \cite{DBLP:conf/acl/WangWGD18,DBLP:conf/pkdd/MinerviniCMNV17,2017Adversarial}  leveraging external background knowledge to achieve the regularization.
\cite{DBLP:conf/acl/GuoWWWG15} enforces the embedding space to be semantically smooth, i.e., entities belonging to the same semantic category will lie close to each other in the embedding space.
\cite{DBLP:conf/ijcai/XieLS16} proposes that
entities should have multiple representations in different types.
In order to impose the prior belief for the structure in the embeddings space, \cite{DBLP:conf/acl/WangWGD18} imposes approximate entailment constraints on relation embeddings and non-negativity constraints on entity embeddings. Recently, \cite{DBLP:conf/nips/ZhangC020} proposes the duality-induced regularizer for TFB models. However, it 
is derived from  the score function of DB models, so
it cannot significantly improve the performance of DB model. 

Based on the  survey above, so far, we can see that previous regularizers  pay more attention to constraints of individual entities or relations rather than the latent semantics of the interaction between entities.  Therefore, we attempt to design a regularizer applied for both DB and TFB models  by employing the potential relations.


\section{Methodology}\label{Methodology}
In this section, we introduce a novel regularizer called  Equivariance Regularizer (\textbf{ER}) for tensor
factorization and distance based KGC models.
Before introducing the methodology, we first provide a brief review of  the knowledge graph completion.
After that, we introduce the ER based on the equivariance of proximity 
 and dissimilarity respectively, in which we
explain their working mechanism. 
  Then we propose the overall ER by integrating proximity and dissimilarity, and futher apply ER to knowledge graphs without entity categories.
Finally, we give
 a theoretical analysis  for the reformulation of ER. 

\textbf{Knowledge graph completion (KGC).} Given a set $\mathcal{E}$ of entities
and a set $\mathcal{R}$ of relations, a knowledge graph $\mathcal{G}\subseteq \mathcal{E}\times \mathcal{R}\times\mathcal{E}$ is  a set of subject-predicate-object triples generally. A KGC model aims to predict missing links between entities based on known links automatically in the knowledge graph. For the convenience of expression, 
 we denote embeddings of entities 
and relations 
as  $\bm{x}_i$ (head entity) or $\bm{x}_k (\text{tail entity})\in \mathbb{R}^{d_e}$ and $\bm{R}_j\in\mathbb{R}^{d_r\times d_r}$ in a low-dimensional vector space, respectively; here $d_e,d_r\in\mathbb{N}^+$ are  the embedding size. Each particular model uses
a scoring function $s$ to associate a score $s(i,k,j)$ with each potential triple
$(i,j,k)\in\mathcal{E}\times\mathcal{R}\times\mathcal{E}$. Then we have the scoring function as $s(i,j,k) = f(\bm{x}_i,\bm{R}_j,\bm{x}_k)$. 
The higher the score of a triple is, the more likely it is considered to be true
by the model.
Our goal is to predict
true but unobserved triples based on the information in $\mathcal{G}$.

Generally, the score functions of the TFB and DB models are as follows, respectively:
$f_{j}^{TFB}(i,j, k)=\boldsymbol{\operatorname { R e }}\left(\overline{\boldsymbol{x}}_{i} \boldsymbol{R}_{j} \boldsymbol{x}_{k}^{\top}\right)=\boldsymbol{\operatorname { R e }}\left(\left\langle\boldsymbol{x}_{i} \overline{\boldsymbol{R}}_{j},
 \boldsymbol{x}_{k}\right\rangle\right),
  f_{j}^{DB}(i, j,k)=-\left\|\boldsymbol{x}_{i} \mathbf{R}_{j}-\boldsymbol{x}_{k}\right\|,$
where $\bm{Re}(\cdot)$ denotes the real
part of the complex number, and $\overline{\cdot}$ denotes the conjugate of a complex matrix.

Notice that the relation-specific matrices play important roles in modeling relations but meanwhile they may also carry some side effects.
On one hand, the relation-specific matrices in TFB and DB models contain lots of parameters, which makes modeling complex relations available. On the other hand, increasing the model complexity also bears the risk of overfitting.

Researchers have proposed a series of regularizers to suppress overfitting.
The basic paradigm of the regularized formulation is  as follows: 
\begin{equation}
\min \sum_{(\bm{x}_i,\bm{R}_j,\bm{x}_k)\in \mathcal{S}} L\left(\bm{x}_i,\bm{R}_j,\bm{x}_k\right)+\lambda g(\bm{x}_i,\bm{R}_j,\bm{x}_k),
\end{equation}
where $L\left(\cdot\right)$ represents the loss function in KGs, $\mathcal{S}$ is the set of triples in KGs, $\lambda$ is a fixed parameter,  and $g(\cdot)$ is
the regularization function.


\subsection{Proximity Based Equivariance}\label{SARA based on Semantic Relevance }

Notice that the previous methods (e.g., \cite{DBLP:conf/acl/GuoWWWG15},  \cite{DBLP:conf/icml/LiuWY17} \cite{DBLP:conf/ijcai/XieLS16}, \cite{DBLP:journals/ws/PadiaKFF19},
\cite{DBLP:conf/pkdd/MinerviniCMNV17}) pay more attention to explicit relational constraints, such as limiting the space of entities and relations. Unfortunately, the  potential relational constraints such as multiple-hop space are  not taken seriously. Therefore, we aim to utilize potential relations based on this perspective.

 Let us take an example.
As shown in Figure \ref{regularizer illustration}, \textbf{London} and \textbf{Paris} have  similar semantics since they are all subordinate to the concept of capital (i.e., \textbf{Paris} is the capital of \textbf{France} and \textbf{London} is the capital of \textbf{England}).  In this way, we expect that \textbf{France} and \textbf{England} linked by  the same relation (\textbf{capital$\_$of}) should have similar semantics, since they are all subordinate to the concept of country. This suggests that embedding proximity is equivariant across head and tail entities, as expressed in the following assumption:
\begin{assumption}
If $\bm{x}_i$ and $\bm{x}_k$ have similar semantics ($\boldsymbol{x}_k \stackrel{near}{\longrightarrow}\boldsymbol{x}_i$),  the entities $\bm{x}_i\boldsymbol{\overline{R}}_j$ and $\bm{x}_k\boldsymbol{\overline{R}}_j$ 
which are produced by $\bm{x}_i$ and $\bm{x}_k$ linking to the same relation $\boldsymbol{R}_j$,  satisfy the equivariance of proximity ($\bm{x}_k\boldsymbol{\overline{R}}_j\stackrel{near}{\longrightarrow}\bm{x}_i\boldsymbol{\overline{R}}_j$).
\end{assumption}



Motivated by this,  we attempt to leverage the latent semantic relations in KGs to tackle the overfitting problem.
First, 
 for any entities  $\boldsymbol{x}_i$ and $\boldsymbol{x}_k$ (e.g., \textbf{London} and \textbf{Paris}) which have similar semantics,  we have the \textbf{norm constraint}: 
\begin{equation}\label{other constraint}
\mathcal{R}_1=||\boldsymbol{x}_i||+||\boldsymbol{x}_k||,
\end{equation}
where $||\cdot||$ can represent any norm of the vector.  The function of the norm constraint here is to limit the embedding of entities in a similar semantic space.

   Recall the given triples $(\bm{x}_i,\bm{R}_{j_1},\bm{x}_p)$ and $(\bm{x}_k, \bm{R}_{j_1},\bm{x}_q)$  in the example above. 
  We can naturally deduce the the equivariance of proximity of $\boldsymbol{x}_i\overline{\mathbf{R}}_{j_1}$ (\textbf{England}) and $\boldsymbol{x}_k\overline{\mathbf{R}}_{j_1}$ (\textbf{Franch}) from $\bm{x}_i$ (\textbf{London}) and $\bm{x}_k$ (\textbf{Paris}) through the same relation $\mathbf{R}_{j_1}$ (\textbf{capital$\_$of}).
  Then we can seek  such a semantics constraint on them to take into account the equivariance of proximity, 
   which we name the \textbf{proximity equivariance constraint}:
\begin{equation}\label{two-order Semantic Relevance}
\mathcal{R}_2=\sum_{i=1}^n\sum_{k=1}^n||\boldsymbol{x}_i\overline{\mathbf{R}}_{j_1}-\boldsymbol{x}_k\overline{\mathbf{R}}_{j_1}||a_{ik},
\end{equation}
where $c_{\bm{x}_i}$ and $c_{\bm{x}_k}$ are the category labels of entities $\bm{x}_i$ and $\bm{x}_k$   respectively, and
\begin{equation}\label{a_ik Semantic relevance}
  a_{ik}=\left\{\begin{array}{ll}
           1,if\  c_{\bm{x}_i}=c_{\bm{x}_k},\\
  0,otherwise.
         \end{array}\right.
\end{equation}

The function of the equivariant proximity constraint is to limit the embedding of link-derived entities to locate in a similar semantic space.
\begin{figure*}
\centering
\subfigure[A KG.]{
\begin{minipage}[t]{0.33\linewidth}
\centering
\includegraphics[width=1.8in]{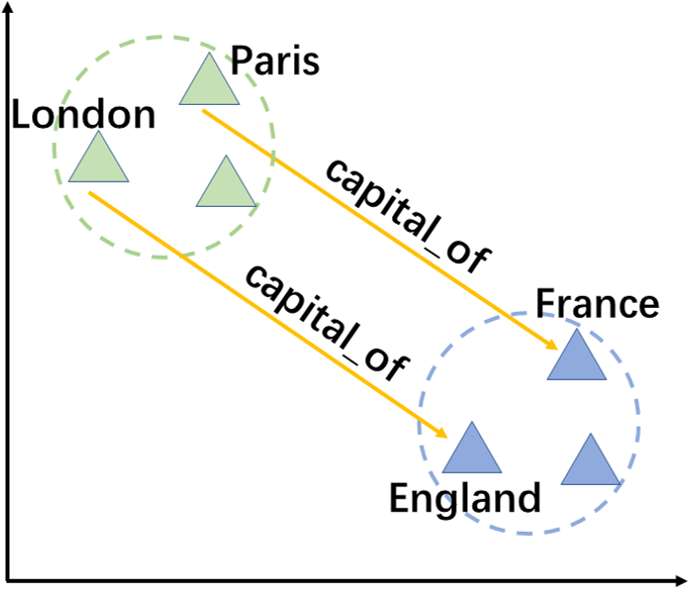}
\label{PIC_KG}
\end{minipage}%
}%
\subfigure[ Similar embeddings.]{
\begin{minipage}[t]{0.33\linewidth}
\centering
\includegraphics[width=1.8in]{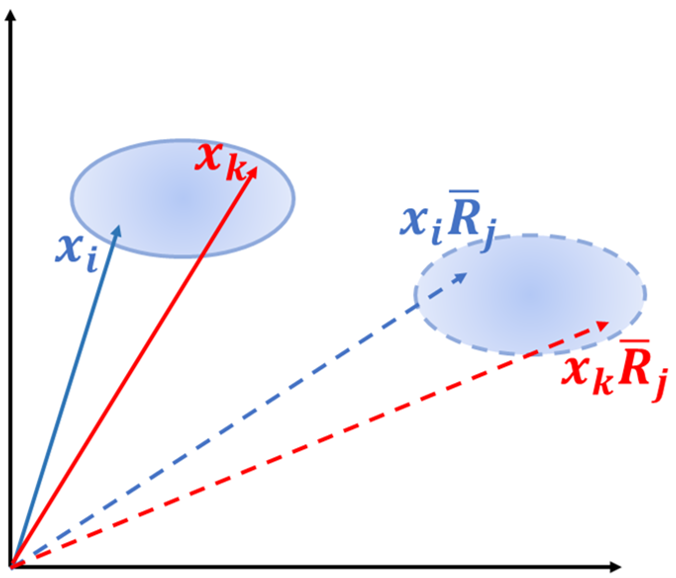}
\label{PIC_DURA}
\end{minipage}
}%
\subfigure[ER.]{
\begin{minipage}[t]{0.33\linewidth}
\centering
\includegraphics[width=1.8in]{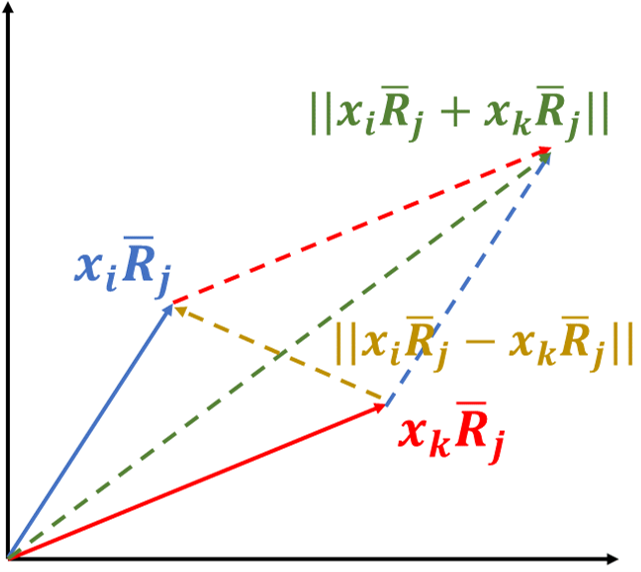}
\label{PIC_MARA}
\end{minipage}
}%
\caption{ An illustration of why ER can improve the performance of KGC models when
the embedding dimensions are 2.
Figure \ref{PIC_KG} shows that there are two triples (London, capital$\_$of, England) denoted as $(\boldsymbol{x}_i,\boldsymbol{R}_j,\boldsymbol{x}_p)$ and (Paris, capital$\_$of, France) denoted as $(\boldsymbol{x}_k,\boldsymbol{R}_j,\boldsymbol{x}_q)$.  Figure \ref{PIC_DURA} shows that  if two entities are  similar in semantics, the other entities connected to them through the same relation separately  satisfy
equivariant proximity in semantics.  Figure \ref{PIC_MARA} shows that ER can realize regularization function with $\|\boldsymbol{x}_i\overline{\mathbf{R}}_j-\boldsymbol{x}_k\overline{\mathbf{R}}_j||$ based on the equivariance of proximity, and  realize regularization function with $\|\boldsymbol{x}_i\overline{\mathbf{R}}_j+\boldsymbol{x}_k\overline{\mathbf{R}}_j||$ based on the equivariance of dissimilarity, respectively.
} \label{regularizer illustration}
\end{figure*}
Then the new model performs the embedding task by minimizing the
following objective function:
\begin{equation}\label{formular for SARA based on Semantic Relevance}
 \min \sum_{(\bm{x}_i,\bm{R}_j,\bm{x}_k)\in \mathcal{S}}
 L\left(\bm{x}_i,\bm{R}_j,\bm{x}_k\right)+\lambda (\mathcal{R}_1+\mathcal{R}_2).
\end{equation}

\subsection{Dissimilarity Based Equivariance}\label{SARA based on Semantic Irrelevance}

We have proposed ER based on the equivariance of proximity  between entities in the above. From another perspective, notice that not all head entities  in triplets are semantically similar, i.e., some entities are semantically distant  in the knowledge graph. Motivated by this, we will give another expression of ER based on the equivariance of dissimilarity.

 Consider such an example. Given two triples (\textbf{London}, \textbf{attribute$\_$of}, \textbf{city}) and (\textbf{Mississippi River}, \textbf{attribute$\_$of}, \textbf{river}),  $\bm{x}_i (\textbf{London})$ \text{and} $\bm{x}_k$ (\textbf{Mississippi River}) are not close in semantics, i.e., they are semantically distant. Since they  have a public relation $\bm{R}_{j_1}$ (\textbf{attribute$\_$of}), we can  infer that  $\boldsymbol{x}_i\overline{\mathbf{R}}_{j_1}$ (\textbf{city}) and $\boldsymbol{x}_k\overline{\mathbf{R}}_{j_1}$ (\textbf{river}) satisfy the equivariance of dissimilarity in semantics.
 This suggests
that embedding dissimilarity is equivariant across head and tail
entities, as expressed in the following assumption:

 \begin{assumption}
 if $\boldsymbol{x}_i$ and $\boldsymbol{x}_k$   are semantically distant $\boldsymbol{x}_k (\stackrel{far}{\longrightarrow}\boldsymbol{x}_i$),   
 the entities $\bm{x}_i\boldsymbol{\overline{R}}_j$ and $\bm{x}_k\boldsymbol{\overline{R}}_j$, 
which are produced by $\bm{x}_i$ and $\bm{x}_k$ linking to the same relation $\boldsymbol{R}_j$,
 should be equivariant dissimilarity ($\bm{x}_k\boldsymbol{\overline{R}}_j\stackrel{far}{\longrightarrow}\bm{x}_i\boldsymbol{\overline{R}}_j$).
 \end{assumption}


First of all, the semantic space of the entities should be bounded to ensure that it will not diverge, so we have the \textbf{norm cosntraint} $\mathcal{W}_1$ same as Eq.(\ref{other constraint}).
 Then we can notice such a fact as shown in Figure \ref{PIC_MARA}: the more futher the semantics of $\boldsymbol{x}_i\overline{\mathbf{R}}_{j_1}$ and $\boldsymbol{x}_k\overline{\mathbf{R}}_{j_1}$ are away from each other, the larger $||\boldsymbol{x}_i\overline{\mathbf{R}}_{j_1}-\boldsymbol{x}_k\overline{\mathbf{R}}_{j_1}||$  is, which means that $||\boldsymbol{x}_i\overline{\mathbf{R}}_{j_1}+\boldsymbol{x}_k\overline{\mathbf{R}}_{j_1}||$ will be smaller in the vector parallelogram for this case. We denote the labels of entity $\bm{x}_i$ and $\bm{x}_k$ as $c_{\bm{x}_i}$ and $c_{\bm{x}_k}$ respectively.
Then we can use the following constraints as the \textbf{dissimilarity equivariance constraint}:
\begin{equation}\label{two-order  Semantic Irrelevance}
\mathcal{W}_2=\sum_{i=1}^n\sum_{k=1}^n||\boldsymbol{x}_i\overline{\mathbf{R}}_{j_1}+\boldsymbol{x}_k\overline{\mathbf{R}}_{j_1}||b_{ik},
\end{equation}
where
\begin{equation}\label{a_ik Semantic irrelevance}
  b_{ik}=\left\{\begin{array}{ll}
           0,\ if\  c_{\bm{x}_i}=c_{\bm{x}_k},\\
           1,\ otherwise.
         \end{array}\right.
\end{equation}
The objective function  could then be formulated as follows:
\begin{equation}\label{formular for SARA based on Semantic Relevance}
 \min \sum_{(\bm{x}_i,\bm{R}_j,\bm{x}_k)\in \mathcal{S}}
 L\left(\bm{x}_i,\bm{R}_j,\bm{x}_k\right)+\lambda (\mathcal{W}_1+\mathcal{W}_2).
\end{equation}

\subsection{Joint Model Based Equivariance}
In the above, we have applied ER to the knowledge graphs with  predefined entity
types. However,
in fact, the entities in some knowledge graphs are not classified into different categories, so we cannot utilize their information about the entity type in this case. Therefore it motivates us to attempt to use a appropriate metric equivalently instead of using entity categories. Based on the  observations above, we have the following
assumption:
\begin{assumption}
If two entities $\bm{x}_i$ and $\bm{x}_k$ have similar embeddings ($\boldsymbol{x}_k \stackrel{near}{\longrightarrow}\boldsymbol{x}_i$), 
they may belong to the same semantic category ($c_{\bm{x}_i}=c_{\bm{x}_k}$).
\end{assumption}

To evaluate  the similarity of two entities $\bm{x}_i$ and $\bm{x}_k$, we propose a  discriminant for their embeddings approximately:
\begin{equation}\label{epsilon hyperparameter}
\begin{aligned}
\left\{
\begin{array}{c}
    c_{\bm{x}_i}=c_{\bm{x}_k}, if\ ||\bm{x}_i-\bm{x}_k||\leq \epsilon_j, \\
   c_{\bm{x}_i}\neq c_{\bm{x}_k}, if\ ||\bm{x}_i-\bm{x}_k|| > \epsilon_j,
\end{array}\right.
\end{aligned}
\end{equation}
where  $\epsilon_j$ is the semantic-similarity parameter for each relation $\bm{R}_j$, which can be learned automatically in ER.  Combining Eq.(\ref{epsilon hyperparameter}) with Eq.(\ref{a_ik Semantic relevance}) and Eq.(\ref{a_ik Semantic irrelevance}), we can apply ER to knowledge graphs without entity categories.  Specifically, if $||\bm{x}_i-\bm{x}_k||\leq \epsilon_j$, we can use ER based on proximity equivariance; otherwise, we can use ER based on dissimilarity equivariance.

Moreover, we notice that there are some instances which does not fit the equivariance of
proximity and dissimilarity. For example, two entities are close in semantics but their linking entities are not so close in semantics. Fortunately, ER  can adjust  $\epsilon_j$ appropriately  larger automatically   in order to make a trade-off. In this case, it can  take into account differences of entities while capturing similarities.

In general, ER can handle knowledge graphs with and without category labels of entities,  which provides a theoretical guarantee that the ER  can be widely used in different knowledge graphs.  Moreover,  we will see ER can be widely adopted in various KGC models and achieve the competitive performance in the experimental part.

\subsection{Theoretic Analysis for ER}\label{Theoretic Analysis}


We have derived the general form of the regularizer based on the equivariance of proximity and dissimilarity, respectively. 
Next, we take the some models
using ER as an example to give the reformulation of the regularizer. 
 First of all, we define the nuclear t-norm of a 3D tensor. 

\begin{definition}
In the knowledge graph completion problem, the tensor nuclear t-norm of $\hat{\mathcal{X}}$ is
\begin{equation}\label{t-noem of X}
\begin{aligned}
\small{
\|\hat{\mathcal{X}}\|_{t}}=\min &\left\{\sum_{d=1}^{D}\left\|\mathbf{p}_{: d}\right\|_{t}\left\|\mathbf{r}_{: d}\right\|_{t}\left\|\mathbf{q}_{: d}\right\|_{t}:\right.\\
& \left.\hat{\mathcal{X}}=\sum_{d=1}^{D} \mathbf{p}_{: d} \otimes \mathbf{r}_{: d} \otimes \mathbf{q}_{: d}\right\},
\end{aligned}
\end{equation}
where $D$ is the embedding dimension, $\bm{p}_{:d}$, $\bm{r}_{:d}$, and $\bm{q}_{:d}$ denote the $d$-th columns of $\bm{x}_i$, $\bm{R}_j$, and $\bm{x}_k$.
\end{definition}
Then we have the following reformulations of ER based on the equivariance for the case of 2-norm. The conclusion for other models  can be analogized accordingly.
\begin{theorem}\label{Theorem1}
Suppose that the model mechanism  $\hat{\bm{X}}=\bm{x}_i\overline{\bm{R}}_j\bm{x}^{\top}_k$ for $j=1,2,\cdots,|\mathcal{R}|$, where $\bm{x}_i,\bm{x}_k$ are real matrices, $\bm{R}_j$ is diagonal. Denote $\mathcal{S}$ as the set of triples in KGs. Then, the following equation holds
\begin{equation}\label{bound1}
\begin{aligned}
\frac{1}{ \sqrt{|\mathcal{R}|}}\small{\min\sum_{(\boldsymbol{x}_i,\boldsymbol{R}_j,\boldsymbol{x}_k)\in\mathcal{S}}}||\boldsymbol{x}_i||_F^2+||\boldsymbol{x}_k||_F^2+\|\boldsymbol{x}_i\mathbf{R}_j-\boldsymbol{x}_k\mathbf{R}_j||_F^2
\\=\|\hat{\mathcal{X}}\|_{2}
\end{aligned}
\end{equation}
\end{theorem}
\begin{proof}
Please see the Appendix\ref{Proof1t}\footnote{https://github.com/Lion-ZS/ER}.
\end{proof}
\begin{theorem}\label{Theorem2} 
Suppose that the model mechanism $\hat{\bm{X}}=\bm{x}_i\overline{\bm{R}}_j-\bm{x}_k$ for $j=1,2,\cdots,|\mathcal{R}|$, where $\bm{x}_i,\bm{x}_k,\bm{R}_j$ are real matrices and
$\bm{R}_j$ is diagonal. Then, the following equation holds
\begin{equation}\label{bound2}
\begin{aligned}
\small{\frac{1}{2 \sqrt{|\mathcal{R}|}}\min\sum_{(\boldsymbol{x}_i,\boldsymbol{R}_j,\boldsymbol{x}_k)\in\mathcal{S}}||\boldsymbol{x}_i||_F^2+||\boldsymbol{x}_k||_F^2+\|\boldsymbol{x}_i\mathbf{R}_j+\boldsymbol{x}_k\mathbf{R}_j||_F^2
}\\=\|\hat{\mathcal{X}}\|_{2}
\end{aligned}
\end{equation}
\end{theorem}
\begin{proof}
Please see the Appendix\ref{Proof2t}.
\end{proof}

We have discussed the reformulation of the regularizer with the 2-norm ($t=2$ in Eq.(\ref{t-noem of X})) above. 
It shows the relationship between  ER  and the tensor nuclear 2-norm regularizer.
Notice that the norm-based regularization has been extensively studied in the context of matrix completion. The trace norm (or nuclear norm) has been proposed as a convex relaxation of the rank \cite{DBLP:conf/nips/SrebroRJ04} for matrix completion in the setting of rating prediction, with strong theoretical guarantees \cite{DBLP:journals/focm/CandesR09}.
  Then the weighted nuclear 3-norm $||\cdot||_3$ can be easily implemented by keeping the regularization terms corresponding to the sampled triples only.
Therefore, based on the equivariance of proximity and dissimilarity  respectively,  we introduce reformulations of the regularizer with 3-norm ($t=3$ in Eq.(\ref{t-noem of X})).
\begin{theorem}\label{Theorem3} 
Suppose that the model mechanism $\hat{\bm{X}}=\bm{x}_i\overline{\bm{R}}_j\bm{x}^{\top}_k$ for $j=1,2,\cdots,|\mathcal{R}|$, where $\bm{x}_i,\bm{x}_k,\bm{R}_j$ are real matrices and
$\bm{R}_j$ is diagonal. Then, the following equation holds
\begin{equation}\label{bound3}
\begin{aligned}
\small{\frac{1}{ \sqrt{|\mathcal{R}|}}\min\sum_{(\boldsymbol{x}_i,\boldsymbol{R}_j,\boldsymbol{x}_k)\in\mathcal{S}}||\boldsymbol{x}_i||_3^3+||\boldsymbol{x}_k||_3^3+\|\boldsymbol{x}_i\mathbf{R}_j-\boldsymbol{x}_k\mathbf{R}_j||_3^3
}\\
=\|\hat{\mathcal{X}}\|_{3}
\end{aligned}
\end{equation}
\end{theorem}
\begin{proof}
Please see the Appendix\ref{Proof3t}.
\end{proof}

\begin{theorem}\label{Theorem4} 
Suppose that the model mechanism $\hat{\bm{X}}=\bm{x}_i\overline{\bm{R}}_j-\bm{x}_k$ for $j=1,2,\cdots,|\mathcal{R}|$, where $\bm{x}_i,\bm{x}_k,\bm{R}_j$ are real matrices and
$\bm{R}_j$ is diagonal. Then, the following equation holds
\begin{equation}\label{bound4}
\begin{aligned}
\small{\frac{1}{4 \sqrt{|\mathcal{R}|}}\min\sum_{(\boldsymbol{x}_i,\boldsymbol{R}_j,\boldsymbol{x}_k)\in\mathcal{S}}||\boldsymbol{x}_i||_3^3+||\boldsymbol{x}_k||_3^3+\|\boldsymbol{x}_i\mathbf{R}_j+\boldsymbol{x}_k\mathbf{R}_j||_3^3
}\\=\|\hat{\mathcal{X}}\|_{3}
\end{aligned}
\end{equation}
\end{theorem}
\begin{proof}
Please see the Appendix\ref{Proof4t}.
\end{proof}

We have given the reformulation of ER to the regularizer of 3-norm.  It shows the relationship
between ER and the tensor nuclear 3-norm regularizer. Moreover, we can  see that ER based on the equivariance of proximity is a stronger constraint  than ER based on the equivariance of dissimilarity.  It means that ER based on the equivariance of proximity may be more conducive to improving the performance of the model than ER based on the equivariance of dissimilarity when the overfitting is severe. We will verify this conjecture in the experimental part.

\section{Experiments}\label{Experiments}
In this section, we first introduce the experimental settings 
and show main results. Then we conduct some ablation experiments. 
\subsection{Experimental Settings}\label{Experimental Settings}
\textbf{Datasets Description}: We conduct experiments on three widely used benchmarks, WN18RR \cite{dettmers2017convolutional}, FB15K-237 \cite{dettmers2017convolutional} and YAGO3-10 \cite{2013YAGO3},  of which the statistics are summarized in Table \ref{data}.
WN18RR  is a subset of WN18 \cite{bordes2013translating} and it embraces a hierarchical collection of relations between words. 
FB15K-237  is a subset of FB15K \cite{bordes2013translating}, in which inverse relations are removed. YAGO3-10 is larger in scale. 

\textbf{Evaluation Protocol}: Two popular evaluation metrics are used: Mean
Reciprocal Rank (MRR) and Hit ratio with cut-off values (Hits@n,n = 1,10).  MRR is the average
inverse rank for correct entities. Hit@n measures the proportion of correct entities in the top $n$ entities.
Following \cite{bordes2013translating}, we report the filtered results  to avoid possibly flawed evaluation.

 \textbf{Baselines}: We compared ER with a number of strong baselines.  For DB models, we report  MuRP \cite{2019Multi}
 , RotatE \cite{sun2019rotate:}, QuatE \cite{zhang2019quaternion}, DualE \cite{DBLP:conf/aaai/CaoX0CH21},
 HypER \cite{balavzevic2019hypernetwork}, REFE \cite{chami2020low} and GC-OTE \cite{tang-etal-2020-orthogonal}; For TFB models, we report TuckER \cite{balavzevic2019tucker},  CP \cite{Hitchcock1927The}, DistMult \cite{yang2014embedding}  and ComplEx \cite{trouillon2016complex}. 

\textbf{Implementation Details:}  
Following \cite{conf/iclr/RuffinelliBG20}, we adopt the cross entropy loss function for all compared models.
 We take Adagrad \cite{duchi2011adaptive} as the optimizer in the experiment, and use grid search based on the performance of the validation datasets to choose the best hyperparameters. Specifically, we search learning
rates in $\{0.5,0.1,0.05,0.01,0.005,0.001\}$, and search regularization coefficients in $\{0.001,0.005,0.01,0.05,0.1,0.5\}$.  
All models are trained for a maximum of 200 epochs. 

As shown in \cite{2010Collaborative} and \cite{lacroix2018canonical},  the weighted versions of regularizers will usually have a better performance than the unweighted
versions if we sample entries of the matrix or tensor non-uniformly. Therefore, in the
experiments, we implement ER of 2-norm in a weighted way (note that we all use the version of ER with this setting if there is no special explanations in the experiments).

\begin{table*}
\begin{center}
 \begin{tabular}{lccccccccc}
  \toprule
  &  \multicolumn{3}{c}{WN18RR}&\multicolumn{3}{c}{FB15K-237}&\multicolumn{3}{c}{YAGO3-10}\\
 Models   &MRR &Hits@1 &Hits@10  &MRR &Hits@1 &Hits@10&MRR &Hits@1 &Hits@10\\
  \midrule
  DistMult& .430& .390&  .490& .241& .155& .419& .340& .240& .540\\
  ConvE &.430& .400& .520& .325& .237& .501& .440& .350& .620\\
MuRP &.481 &.440 &.566 &.335 &.243 &.518 &- &- &-\\
TuckER &.470 &.443 &.526 &.358 &.266 &.544 &- &- &-\\
 CP &.438 &.414 &.485 &.333 &.247 &.508 &.567 &.494 &.698\\
RESCAL &.455 &.419 &.493 &.353 &.264 &.528 &.566 &.490 &.701\\
ComplEx &.460 &.428 &.522 &.346 &.256 &.525 &.573 &.500 &.703\\
 RotatE &.476 &.428 &.571 &.338 &.241 &.533 &.495 &.402 &.670\\
 QuatE &.488& .438&  .580& .348& .248& .550& .571 &.501 &.705\\
 DualE &492 &.444 &.581&.365 &.268  &.559&.575&.505&.710\\
   REFE &.473& .430&  .561& .351& .256& .541& .577& .503&  .712\\
 GC-OTE & .491 &.442 & .583& .361 &.267 & .550&-&-&-\\
 ROTH &\underline{.496} &\underline{.449}& $\bm{.586}$& .344 &.246&  .535& .570& .495& .706\\
   \midrule
\text{CP-ER}&.482& .444& .557& .371&.275& .561& \underline{.584}& .508&.712\\
\text{RotatE-ER} &.490 &.445 &.581&.352&.255&.547&.581&.505 &.704\\
\text{RESCAL-ER}& \textbf{.499}& \textbf{.458}& \underline{.582}& \underline{.373}& \underline{.281}& \underline{.554}&.583& \underline{.509} &\underline{.715}\\
 \text{ComplEx-ER} &.494& .453& .575& \textbf{.374}& \textbf{.282}& \textbf{.563}& \textbf{.588}& \textbf{.515}&\textbf{.718}\\
  \bottomrule
 \end{tabular}
  \caption{Evaluation results on WN18RR, FB15k-237 and YAGO3-10 datasets. The best performance on each model are marked in bold and second best
results are underlined.}
\label{all results}
 \end{center}
\end{table*}

\subsection{Results}\label{Results}
we can see that ER is an effective regularizer from Table \ref{all results}.
 Due to the existence of overfitting, the performance of the TFB models is usually not as advantaged as the DB models. However,  using ER, the performance of RESCAL and ComplEx has been significantly improved, and can exceed that of the DB models.


For the three datasets in the table, the size of WN18RR is the smallest since it has only 11 relations and about 86,000 training samples. Generally, if there are  more parameters in a model, it is more likely to cause overfitting when it is applied to a dataset with a smaller size. Therefore, compared with other datasets, we predict that the improvement brought by ER on WN18RR is expected to be greater, and the experiment also proves this point.
On WN18RR dataset, RESCAL gets an MRR score of 0.455, which is lower than ComplEx
($0.460$).  However, incorporated with
ER, RESCAL gets the $4.4\%$ improvement on MRR and finally attains $0.499$, which outperforms
all compared models.
Overfitting problem also exists on larger data sets 
and we can see that  the performance of the model has been significantly improved after utilizing ER. We find that the application of ER can also bring stable and meaningful improvements to the model, which shows that ER can effectively suppress the overfitting.

\subsection{Experiment Analysis}
\textbf{Ablation Study on  Equivariance of Proximity.}
 Recall we have given the expression of ER in the  section of the methodology. Furthermore, given triples $(\boldsymbol{x}_i\overline{\mathbf{R}}_{j_1},\mathbf{R}_{j_2},?)$ and $(\boldsymbol{x}_k\overline{\mathbf{R}}_{j_1},\mathbf{R}_{j_2},?)$, if $\mathbf{R}_{j_2}$ (\textbf{country}\_\textbf{of})  is a relation connected to the entities $\boldsymbol{x}_i\overline{\mathbf{R}}_{j_1}$ (\textbf{England}) and $\boldsymbol{x}_k\overline{\mathbf{R}}_{j_1}$ (\textbf{Franch}) respectively, we expect the embeddings of $\boldsymbol{x}_i\overline{\mathbf{R}}_{j_1}\overline{\mathbf{R}}_{j_2}$ (\textbf{Europe}) and $\boldsymbol{x}_k\overline{\mathbf{R}}_{j_1}\overline{\mathbf{R}}_{j_2}$ (\textbf{Europe}) locate in a similar region. 
Then we can derive the \textbf{second order equivariance of proximity}:
\begin{equation}\label{third-order constraint}
\mathcal{R}_3=||\boldsymbol{x}_i\overline{\mathbf{R}}_{j_1}\overline{\mathbf{R}}_{j_2}-\boldsymbol{x}_k\overline{\mathbf{R}}_{j_1}\overline{\mathbf{R}}_{j_2}||. 
\end{equation}
To take advantage of various semantic information, we integrate Eq.(\ref{third-order constraint}), Eq.(\ref{other constraint}) and Eq.(\ref{two-order Semantic Relevance}), then we propose the regularization  based on the second order equivariance of proximity  as follows:
\begin{equation}\label{g(x)-n}
\begin{aligned}
\min \sum_{(\bm{x}_i,\bm{R}_j,\bm{x}_k)\in \mathcal{S}}
 L\left(\bm{x}_i,\bm{R}_j,\bm{x}_k\right)+
\lambda\sum_{i =1}^3\mathcal{R}_i
\end{aligned}
\end{equation}
where $\mathcal{R}_i$ represents the $i$-th order equivariance of proximity.

In order to explore the impact of $\mathcal{R}_i$, we implement an ablation study on different order equivariance of proximity for ComplEx in FB15K237. As shown in Table \ref{complex1-4} in Appendix, ComplEx$_i$ means ComplEx using $\mathcal{R}_i$. We can see that
 ComplEx$_1$ 
 is significantly higher than ComplEx$_0$. 
With the increase of the semantics constraints, we can see that the effect of ComplEx$_2$ 
is also improving compared to ComplEx$_1$, but the improvements are  not very significant.
 This shows that it may be more appropriate and concise to adopt ER with the firstr order equivariance of proximity as the regularizer.

\textbf{Comparison to Other Regularizers.}
 We compare ER to the squared Frobenius norm (FRO) regularizer, the
tensor nuclear 3-norm regularizer and DURA regularizer. 
As
stated in Table \ref{tensor models} in Appendix, we  shows the performance of the four regularizers on three TFB
models: CP, ComplEx, and RESCAL. 
Table \ref{distance models} in Appendix  shows the performance of four regularizers on two DB
models: TransE  and RotatE.

First of all, 
we can see that ER  achieves the best performance and DURA regularizer achieves the second for TFB models, while ER  achieves the best performance and N3 regularizer achieves the second for DB models.
 It  illustrates that ER is more beneficial for both TFB and DB models than other regularizers. Overall, capturing semantic similarity can produce powerful and consistent improvements to DB and TFB models. It further proves the rationality that ER regularizer takes advantage of semantics  order
constraints.
Specifically, 
for the previous  KGC models such as  ComplEx and RotatE, ER produces further
improvements than the DURA and N3 regularizers. 
Overall,  ER shows obvious advantages over other models by applying latent semantic relations so that it can bring consistent improvements  to KGC models. And the results  also demonstrate that ER is more
 effective than DURA, N3 and FRO regularizers.

\textbf{Comparison to Other Semantic Constraint Models}. We have noticed these models: \cite{DBLP:conf/acl/GuoWWWG15} enforces the embedding space to be semantically smooth, which we denote as  LLE; \cite{DBLP:conf/ijcai/XieLS16} suggests that entities should have multiple representations in different types, which we denote as RHE;
 \cite{DBLP:journals/ws/PadiaKFF19} predicts knowledge graph fact via knowledge-enriched tensor factorization, which we denote as TFR; \cite{DBLP:conf/pkdd/MinerviniCMNV17} regularize knowledge graph embeddings via equivalence and inversion axioms, which we denote as EIA;
 \cite{zhangpretrain}  leverages the abundant linguistic knowledge from pretrained language models, which we denote as Pretrain. Here we compare them with ER as shown in Table \ref{Comparison of semantic information models} in Appendix. Some of them  try to enforce semantically similar entities to have similar embedding vectors or
exploit the analogical structures in a knowledge graph. However, the semantic information they use mostly based on first-order like Eq.(\ref{other constraint}). For the potential deeper semantic relation, such as Eq.(\ref{two-order  Semantic Relevance}) and Eq.(\ref{two-order  Semantic Irrelevance}) which they have not paid attention to. Moreover, some methods above are not suitable for knowledge graphs of unknown entity categories.
Notice pretrain method can learn entity and relation representations
via pretrained language models,
but we can see that ER performs better than pretrain method,  which show the importance of capturing the potential semantics in KGC.

\textbf{Study for ER of 3-norm.}
In the  experiments above, we mainly study the ER of 2-norm. Next, we conduct a comparative experiment on ER of 3-norm.
As shown in Table \ref{3-norm MARA with all models} in Appendix, we can see ER of 3-norm can effectively improve the performance of the model, which is similar to the performance of ER of the 2-norm. This demonstrates that ER is a stable regularizer, and it is robust to applications in different norm scenarios. 

\textbf{Study of Training and Validation Curves.}
We conduct the  research for RESCAL on training set and validation set of WN18RR. As shown in Figure \ref{train_validation} in Appendix, we can see that RESCAL without regularizer performs very well on the training set, but it performs poorly on the validation set, which shows that RECAL has a serious overfitting phenomenon. As the number of training epochs increases on the validation set, 
MRR of RESCAL without regularizer stops growing when it reaches around 0.455 while MRR of RESCAL with ER can reach 0.502. It demonstrates that ER can effectively suppresses overfitting  and improve the generalization performance of the model.


\section{Conclusion}\label{Conclusion}
We propose a widely applicable and effective regularizer named Equivariance Regularizer (ER) to suppress overfitting in KGE problem. 
 It is based on the observation that the previous work only focuses on the similarity between entities while
ignoring the latent semantic relation between them. 
 Theoretically, ER can suppress overfitting and benefit the expressive ability of
the model by utilizing the potential semantic information based on the equivariance of proximity and dissimilarity. 
All the results demonstrate that ER is more
widely applicable than other regularizers.
Experiments show that ER brings consistent and significant
improvements to TFB and DB models on benchmark datasets.

\section{Acknowledgements}
This work was supported by the National Key R$\&$D Program of China under Grant 2018AAA0102003.

\bibliography{SARA} 
\newpage
\clearpage

\newpage
\section{Appendix}

\appendix

\section{Proof for Theorem \ref{Theorem1}}\label{Proof1t}

\textbf{Theorem 1}: Suppose that $\hat{\bm{X}}=\bm{x}_i\overline{\bm{R}}_j\bm{x}^{\top}_k$ for $j=1,2,\cdots,|\mathcal{R}|$, where $\bm{x}_i,\bm{x}_k,\bm{R}_j$ are real matrices and
$\bm{R}_j$ is diagonal. Then, the following equation holds
\begin{equation*}
\begin{aligned}
\small{\min}\frac{1}{\sqrt{|\mathcal{R}|}}\sum_{(\boldsymbol{x}_i,\boldsymbol{R}_j,\boldsymbol{x}_k)\in\mathcal{S}}||\boldsymbol{x}_i||_F^2+||\boldsymbol{x}_k||_F^2+\|\boldsymbol{x}_i\mathbf{R}_j-\boldsymbol{x}_k\mathbf{R}_j||_F^2
\\=\|\hat{\mathcal{X}}\|_{2}.
\end{aligned}
\end{equation*}\begin{proof}
Notice that
\begin{equation}
\begin{aligned}
&\min\sum_{(\boldsymbol{x}_i,\boldsymbol{R}_j,\boldsymbol{x}_k)\in\mathcal{S}}||\boldsymbol{x}_i||_F^2+||\boldsymbol{x}_k||_F^2+\|\boldsymbol{x}_i\mathbf{R}_j-\boldsymbol{x}_k\mathbf{R}_j||_F^2\\
\stackrel{a}{\leq}&\min\sum_{(\boldsymbol{x}_i,\boldsymbol{R}_j,\boldsymbol{x}_k)\in\mathcal{S}}||\boldsymbol{x}_i||_F^2+||\boldsymbol{x}_k||_F^2+\|\boldsymbol{x}_i\mathbf{R}_j\|_F^2+\|\boldsymbol{x}_k\mathbf{R}_j||_F^2\\&-\|\boldsymbol{x}_i\mathbf{R}_j\boldsymbol{x}_k\mathbf{R}_j\|-\|\boldsymbol{x}_k\mathbf{R}_j\boldsymbol{x}_i\mathbf{R}_j\|\\
\leq & \min\sum_{(\boldsymbol{x}_i,\boldsymbol{R}_j,\boldsymbol{x}_k)\in\mathcal{S}}||\boldsymbol{x}_i||_F^2+||\boldsymbol{x}_k||_F^2+\|\boldsymbol{x}_i\mathbf{R}_j\|_F^2+\|\boldsymbol{x}_k\mathbf{R}_j||_F^2\\
&-2\|\boldsymbol{x}_i\mathbf{R}_j\boldsymbol{x}_k\mathbf{R}_j\|_F^2\\
\leq & \min\sum_{(\boldsymbol{x}_i,\boldsymbol{R}_j,\boldsymbol{x}_k)\in\mathcal{S}}||\boldsymbol{x}_i||_F^2+||\boldsymbol{x}_k||_F^2+\|\boldsymbol{x}_i\mathbf{R}_j\|_F^2+\|\boldsymbol{x}_k\mathbf{R}_j||_F^2
.\end{aligned}
\end{equation}
 Since  $\boldsymbol{x}_i \boldsymbol{R}_{j}$ and $\boldsymbol{x}_k \boldsymbol{R}_{j}$  are all vectors, we can have $\|\boldsymbol{x}_i \boldsymbol{R}_{j}\boldsymbol{x}_k \boldsymbol{R}_{j}\|=\|\boldsymbol{x}_k \boldsymbol{R}_{j}\boldsymbol{x}_i \boldsymbol{R}_{j}\|$. Then the inequality (a) holds.

We first prove that the following equation holds
\begin{equation}\label{center formular}
\begin{aligned}
\min _{\hat{X}_{j}= \boldsymbol{x}_i R_{j} x_k^{\top}} \frac{1}{\sqrt{|\mathcal{R}|}} \sum_{j=1}^{|\mathcal{R}|}\left(\|\boldsymbol{x}_i\|_{F}^{2}+\|\boldsymbol{x}_k\|_{F}^{2}+\left\|\boldsymbol{x}_i\boldsymbol{R}_{j}\right\|_{F}^{2}\right.\\
\left.+\|\boldsymbol{x}_k \boldsymbol{R}_{j}\|_{F}^{2}\right)=\|\hat{\mathcal{X}}\|_{2}.
\end{aligned}
\end{equation}
Denote $\boldsymbol{x}_i$ as $\boldsymbol{p}$ and $\boldsymbol{x}_k$ as $\boldsymbol{q}$. Then we have that
\begin{equation*}
\begin{aligned}
& \sum_{j=1}^{|\mathcal{R}|}\left(\left\|\boldsymbol{x}_i \mathbf{R}_{j}\right\|_{F}^{2}+\|\boldsymbol{x}_k\|_{F}^{2}\right) \\
=& \sum_{j=1}^{|\mathcal{R}|}\left(\sum_{d=1}^{D}\left\|\mathbf{q}_{: d}\right\|_{F}^{2}+\sum_{i=1}^{I} \sum_{d=1}^{D} \mathbf{p}_{i d}^{2} \mathbf{r}_{j d}^{2}\right) \\
=& \sum_{j=1}^{|\mathcal{R}|} \sum_{d=1}^{D}\left\|\mathbf{q}_{: d}\right\|_{2}^{2}+\sum_{d=1}^{D}\left\|\mathbf{p}_{: d}\right\|_{2}^{2}\left\|\mathbf{r}_{: d}\right\|_{2}^{2} \\
=& \sum_{d=1}^{D}\left(\left\|\mathbf{p}_{: d}\right\|_{2}^{2}\left\|\mathbf{r}_{: d}\right\|_{2}^{2}+|\mathcal{R}|\left\|\mathbf{q}_{: d}\right\|_{2}^{2}\right) \\
\geq & \sum_{d=1}^{D} 2 \sqrt{|\mathcal{R}| } \| \mathbf{p}_{: d}\left\|_{2}\right\| \mathbf{r}_{: d}\left\|_{2}\right\| \mathbf{q}_{: d} \|_{2} \\
=& 2 \sqrt{|\mathcal{R}|} \sum_{d=1}^{D}\left\|\mathbf{p}_{: d}\right\|_{2}\left\|\mathbf{r}_{: d}\right\|_{2}\left\|\mathbf{q}_{: d}\right\|_{2}
\end{aligned}.
\end{equation*}
We can have the equality holds if and only if $\left\|\mathbf{p}_{: d}\right\|_{2}^{2}\left\|\mathbf{r}_{: d}\right\|_{2}^{2}=|\mathcal{R}|\left\|\mathbf{q}_{: d}\right\|_{2}^{2},$ i.e., $\left\|\mathbf{p}_{: d}\right\|_{2}\left\|\mathbf{r}_{: d}\right\|_{2}=\sqrt{|\mathcal{R}|}\left\|\mathbf{q}_{: d}\right\|_{2}$.

For all CP decomposition $\hat{\mathcal{X}}=\sum_{d=1}^{D} \mathbf{p}_{: d} \otimes \mathbf{r}_{: d} \otimes \mathbf{q}_{: d}$, we can always let $\bm{p}'_{:d}=\bm{p}_{:d}$, $\bm{r}'_{:d}=\sqrt{\frac{\left\|\mathbf{q}_{d}\right\|_{2} \sqrt{|\mathcal{R}|}}{\left\|\mathbf{p}_{: d}\right\|_{2}\left\|\mathbf{r}_{: d}\right\|_{2}} }\mathbf{r}_{: d}$ and
$\bm{q}'_{:d}=\sqrt{\frac{\left\|\mathbf{p}_{: d}\right\|_{2}\left\|\mathbf{r}_{: d}\right\|_{2}}{\left\|\mathbf{q}_{\cdot d}\right\|_{2} \sqrt{|\mathcal{R}|}}} \mathbf{p}_{: d}$, then we have
\begin{equation*}
\left\|\mathbf{p}_{: d}^{\prime}\right\|_{2}\left\|\mathbf{r}_{: d}^{\prime}\right\|_{2}=\sqrt{|\mathcal{R}|}\left\|\mathbf{q}_{: d}^{\prime}\right\|_{2}
,
\end{equation*}
and at the same time we  make sure that $\hat{\mathcal{X}}=\sum_{d=1}^{D} \mathbf{p}_{: d}^{\prime} \otimes \mathbf{r}_{: d}^{\prime} \otimes \mathbf{q}_{: d}^{\prime}$. Then, we can have
\begin{equation*}
\begin{aligned}
\small{\frac{1}{\sqrt{|\mathcal{R}|}}} \sum_{j=1}^{|\mathcal{R}|}\left\|\hat{\mathcal{X}}_{j}\right\|_{2} &=\frac{1}{2 \sqrt{|\mathcal{R}|}} \sum_{j=1}^{|\mathcal{R}|} \min _{\hat{\mathcal{X}}_{j}=\boldsymbol{x}_i\mathbf{R}_{j}\boldsymbol{x}_k^{\top}}\left(\left\|\boldsymbol{x}_i \mathbf{R}_{j}\right\|_{F}^{2}+\|\boldsymbol{x}_k\|_{F}^{2}\right) \\
& \leq \frac{1}{2 \sqrt{|\mathcal{R}|}} \min _{\hat{\mathcal{X}}_{j}=\boldsymbol{x}_i\mathbf{ R}_{j} \boldsymbol{x}_k^{\top}} \sum_{j=1}^{|\mathcal{R}|}\left(\left\|\boldsymbol{x}_i \mathbf{R}_{j}\right\|_{F}^{2}+\|\boldsymbol{x}_k\|_{F}^{2}\right) \\
&=\min_{\hat{\mathcal{X}}=\sum_{d=1}^{D} \mathbf{p}_{j} \otimes \mathbf{r}_{i d} \otimes \mathbf{q}_{: d}} \sum_{d=1}^{D}\left\|\mathbf{p}_{: d}\right\|_{2}\left\|\mathbf{r}_{: d}\right\|_{2}\left\|\mathbf{q}_{: d}\right\|_{2} \\
&=\|\hat{\mathcal{X}}\|_{2} .
\end{aligned}
\end{equation*}
And in the same manner, we can have that
\begin{equation*}
\begin{aligned}
\small{\min}\frac{1}{\sqrt{|\mathcal{R}|}}\sum_{(\boldsymbol{x}_i,\boldsymbol{R}_j,\boldsymbol{x}_k)\in\mathcal{S}}||\boldsymbol{x}_i||_F^2+||\boldsymbol{x}_k||_F^2+\|\boldsymbol{x}_i\mathbf{R}_j-\boldsymbol{x}_k\mathbf{R}_j||_F^2
\\=\|\hat{\mathcal{X}}\|_{2}.
\end{aligned}
\end{equation*}
We can have the equality holds if and only if $\left\|\mathbf{q}_{: d}\right\|_{2}\left\|\mathbf{r}_{: d}\right\|_{2}=\sqrt{|\mathcal{R}|}\left\|\mathbf{p}_{: d}\right\|_{2}$.
  Then we can see that the conclusion holds if and only if $\left\|\mathbf{p}_{: d}\right\|_{2}\left\|\mathbf{r}_{: d}\right\|_{2}=\sqrt{|\mathcal{R}|}\left\|\mathbf{q}_{: d}\right\|_{2}$ and
$||\bm{q}_{:d}||_2||\bm{r}_{:,d}||_2=\sqrt{|\mathcal{R}|}\left\|\mathbf{p}_{: d}\right\|_{2}, \forall d \in\{1,2, \ldots, D\}$.
Then the proof of Theorem \ref{Theorem1} completes.
\end{proof}
\section{Proof for Theorem \ref{Theorem2}}\label{Proof2t}
\textbf{Theorem 2.} Suppose that $\hat{\bm{X}}=\bm{x}_i\overline{\bm{R}}_j\bm{x}^{\top}_k$ for $j=1,2,\cdots,|\mathcal{R}|$, where $\bm{x}_i,\bm{x}_k,\bm{R}_j$ are real matrices and
$\bm{R}_j$ is diagonal. Then, the following equation holds
\begin{equation*}
\begin{aligned}
\small{\min}\frac{1}{2\sqrt{|\mathcal{R}|}}\sum_{(\boldsymbol{x}_i,\boldsymbol{R}_j,\boldsymbol{x}_k)\in\mathcal{S}}||\boldsymbol{x}_i||_F^2+||\boldsymbol{x}_k||_F^2+\|\boldsymbol{x}_i\mathbf{R}_j+\boldsymbol{x}_k\mathbf{R}_j||_F^2
\\=\|\hat{\mathcal{X}}\|_{2}.
\end{aligned}
\end{equation*}
\begin{proof}
 First we have
\begin{equation*}
\begin{aligned}
& \small{\min}\frac{1}{2\sqrt{|\mathcal{R}|}} \sum_{(\boldsymbol{x}_i,\boldsymbol{R}_j,\boldsymbol{x}_k)\in\mathcal{S}}||\boldsymbol{x}_i||_F^2+||\boldsymbol{x}_k||_F^2+\|\boldsymbol{x}_i\mathbf{R}_j+\boldsymbol{x}_k\mathbf{R}_j||_F^2\\
\leq& \min _{(\boldsymbol{x}_i,\boldsymbol{R}_j,\boldsymbol{x}_k)\in\mathcal{S}} \frac{1}{2\sqrt{|\mathcal{R}|}} \sum_{j=1}^{|\mathcal{R}|}\|\boldsymbol{x}_i\|_{F}^{2}+\|\boldsymbol{x}_k\|_{F}^{2}+\left\|\boldsymbol{x}_i \boldsymbol{R}_{j}\right\|_{F}^{2}\\
&+\|\boldsymbol{x}_k \boldsymbol{R}_{j}\|_{F}^{2}+ \|\boldsymbol{x}_i \boldsymbol{R}_{j}\boldsymbol{x}_k \boldsymbol{R}_{j} \| +\|\boldsymbol{x}_k \boldsymbol{R}_{j} \boldsymbol{x}_i \boldsymbol{R}_{j}\| \\
\stackrel{a}{=} & \min _{(\boldsymbol{x}_i,\boldsymbol{R}_j,\boldsymbol{x}_k)\in\mathcal{S}} \frac{1}{2\sqrt{|\mathcal{R}|}} \sum_{j=1}^{|\mathcal{R}|}\|\boldsymbol{x}_i\|_{F}^{2}+\|\boldsymbol{x}_k\|_{F}^{2}+\left\|\boldsymbol{x}_i \boldsymbol{R}_{j}\right\|_{F}^{2}\\
&+\|\boldsymbol{x}_k \boldsymbol{R}_{j}\|_{F}^{2}+ 2\|\boldsymbol{x}_i \boldsymbol{R}_{j}\boldsymbol{x}_k \boldsymbol{R}_{j} \|\\
\leq & \min _{(\boldsymbol{x}_i,\boldsymbol{R}_j,\boldsymbol{x}_k)\in\mathcal{S}} \frac{1}{2\sqrt{|\mathcal{R}|}} \sum_{j=1}^{|\mathcal{R}|}\|\boldsymbol{x}_i\|_{F}^{2}+\|\boldsymbol{x}_k\|_{F}^{2}+\left\|\boldsymbol{x}_i \boldsymbol{R}_{j}\right\|_{F}^{2}\\
&+\|\boldsymbol{x}_k \boldsymbol{R}_{j}\|_{F}^{2}+ \left\|\boldsymbol{x}_i \boldsymbol{R}_{j}\right\|_{F}^{2}+\|\boldsymbol{x}_k \boldsymbol{R}_{j}\|_{F}^{2}\\
\leq & \min _{(\boldsymbol{x}_i,\boldsymbol{R}_j,\boldsymbol{x}_k)\in\mathcal{S}} \frac{1}{2\sqrt{|\mathcal{R}|}} \sum_{j=1}^{|\mathcal{R}|}\|\boldsymbol{x}_i\|_{F}^{2}+\|\boldsymbol{x}_k\|_{F}^{2}+2\left\|\boldsymbol{x}_i \boldsymbol{R}_{j}\right\|_{F}^{2}\\
&+2\|\boldsymbol{x}_k \boldsymbol{R}_{j}\|_{F}^{2}\\
\leq & \min _{(\boldsymbol{x}_i,\boldsymbol{R}_j,\boldsymbol{x}_k)\in\mathcal{S}} \small{\frac{1}{\sqrt{|\mathcal{R}|}}} \sum_{j=1}^{|\mathcal{R}|}\|\boldsymbol{x}_i\|_{F}^{2}+\|\boldsymbol{x}_k\|_{F}^{2}+\left\|\boldsymbol{x}_i \boldsymbol{R}_{j}\right\|_{F}^{2}+\|\boldsymbol{x}_k \boldsymbol{R}_{j}\|_{F}^{2}\\
\end{aligned}
\end{equation*}
 Since  $\boldsymbol{x}_i \boldsymbol{R}_{j}$ and $\boldsymbol{x}_k \boldsymbol{R}_{j}$  are all vectors, we can have $\|\boldsymbol{x}_i \boldsymbol{R}_{j}\boldsymbol{x}_k \boldsymbol{R}_{j}\|=\|\boldsymbol{x}_k \boldsymbol{R}_{j}\boldsymbol{x}_i \boldsymbol{R}_{j}\|$. Then the equality (a) holds.

Then in the same manner  with Eq.(\ref{center formular}), we can have that
\begin{equation*}
\begin{aligned}
\min\small{\frac{1}{2\sqrt{|\mathcal{R}|}}}\sum_{(\boldsymbol{x}_i,\boldsymbol{R}_j,\boldsymbol{x}_k)\in\mathcal{S}}||\boldsymbol{x}_i||_F^2+||\boldsymbol{x}_k||_F^2+\|\boldsymbol{x}_i\mathbf{R}_j+\boldsymbol{x}_k\mathbf{R}_j||_F^2
\\=\|\hat{\mathcal{X}}\|_{2}.
\end{aligned}
\end{equation*}
We can have the equality holds if and only if $\left\|\mathbf{q}_{: d}\right\|_{2}\left\|\mathbf{r}_{: d}\right\|_{2}=\sqrt{|\mathcal{R}|}\left\|\mathbf{p}_{: d}\right\|_{2}$.
  Then we can see that the conclusion holds if and only if $\left\|\mathbf{p}_{: d}\right\|_{2}\left\|\mathbf{r}_{: d}\right\|_{2}=\sqrt{|\mathcal{R}|}\left\|\mathbf{q}_{: d}\right\|_{2}$ and
$||\bm{q}_{:d}||_2||\bm{r}_{:,d}||_2=\sqrt{|\mathcal{R}|}\left\|\mathbf{p}_{: d}\right\|_{2}, \forall d \in\{1,2, \ldots, D\}$.
Then the proof of Theorem \ref{Theorem2} completes.

\end{proof}

\section{Proof for Theorem \ref{Theorem3}}\label{Proof3t}
 Here we denote  $|\bm{x}|^3$ as $||\bm{x}||_3^3$ and denote  $|\bm{x}|^\frac{3}{2}$ as $||\bm{x}||_C$. Then we have  Theorem  \ref{Theorem3} and Theorem  \ref{Theorem4} with their proofs as follows:

\noindent\textbf{Theorem 3} Suppose that $\hat{\bm{X}}=\bm{x}_i\overline{\bm{R}}_j\bm{x}^{\top}_k$ for $j=1,2,\cdots,|\mathcal{R}|$, where $\bm{x}_i,\bm{x}_k,\bm{R}_j$ are real matrices and
$\bm{R}_j$ is diagonal. Then, the following equation holds
\begin{equation}\label{bound3}
\begin{aligned}
\frac{1}{ \sqrt{|\mathcal{R}|}}\small{\min}\sum_{(\boldsymbol{x}_i,\boldsymbol{R}_j,\boldsymbol{x}_k)\in\mathcal{S}}||\boldsymbol{x}_i||_3^3+||\boldsymbol{x}_k||_3^3+\|\boldsymbol{x}_i\mathbf{R}_j-\boldsymbol{x}_k\mathbf{R}_j||_3^3
\\=\|\hat{\mathcal{X}}\|_{3}
\end{aligned}
\end{equation}
\begin{proof}
Notice that
\begin{equation*}
\begin{aligned}
&\min\sum_{(\boldsymbol{x}_i,\boldsymbol{R}_j,\boldsymbol{x}_k)\in\mathcal{S}}||\boldsymbol{x}_i||_3^3+||\boldsymbol{x}_k||_3^3+\|\boldsymbol{x}_i\mathbf{R}_j-\boldsymbol{x}_k\mathbf{R}_j||_3^3\\
\leq & \min\sum_{(\boldsymbol{x}_i,\boldsymbol{R}_j,\boldsymbol{x}_k)\in\mathcal{S}}||\boldsymbol{x}_i||_3^3+||\boldsymbol{x}_k||_3^3+\|\boldsymbol{x}_i\mathbf{R}_j\|_3^3+\|\boldsymbol{x}_k\mathbf{R}_j||_3^3
.\end{aligned}
\end{equation*}
We first prove that the following equation holds
\begin{equation}\label{center formular1}
\begin{aligned}
\min _{\hat{X}_{j}= \boldsymbol{x}_i R_{j} x_k^{\top}} \small{\frac{1}{\sqrt{|\mathcal{R}|}}} \sum_{j=1}^{|\mathcal{R}|}\left(\|\boldsymbol{x}_i\|_{3}^{3}+\|\boldsymbol{x}_k\|_{3}^{3}+\left\|\boldsymbol{x}_i \boldsymbol{R}_{j}\right\|_{3}^{3}+\|\boldsymbol{x}_k \boldsymbol{R}_{j}\|_{3}^{3}\right)\\=\|\hat{\mathcal{X}}\|_{3}.
\end{aligned}
\end{equation}
Denote $\boldsymbol{x}_i$ as $\boldsymbol{p}$ and $\boldsymbol{x}_k$ as $\boldsymbol{q}$. Then we have that
\begin{equation*}
\begin{aligned}
& \sum_{j=1}^{|\mathcal{R}|}\left(\left\|\boldsymbol{x}_i \mathbf{R}_{j}\right\|_{3}^{3}+\|\boldsymbol{x}_k\|_{3}^{3}\right) \\
=& \sum_{j=1}^{|\mathcal{R}|}\left(\sum_{d=1}^{D}\left\|\mathbf{q}_{: d}\right\|_{3}^{3}+\sum_{i=1}^{I} \sum_{d=1}^{D} \mathbf{p}_{i d}^{3} \mathbf{r}_{j d}^{3}\right) \\
=& \sum_{j=1}^{|\mathcal{R}|} \sum_{d=1}^{D}\left\|\mathbf{q}_{: d}\right\|_{3}^{3}+\sum_{d=1}^{D}\left\|\mathbf{p}_{: d}\right\|_{3}^{3}\left\|\mathbf{r}_{: d}\right\|_{3}^{3} \\
=& \sum_{d=1}^{D}\left(\left\|\mathbf{p}_{: d}\right\|_{3}^{3}\left\|\mathbf{r}_{: d}\right\|_{3}^{3}+|\mathcal{R}|\left\|\mathbf{q}_{: d}\right\|_{3}^{3}\right) \\
\geq & \sum_{d=1}^{D} 2 \sqrt{|\mathcal{R}| } \| \mathbf{p}_{: d}\left\|_{C}\right\| \mathbf{r}_{: d}\left\|_{C}\right\| \mathbf{q}_{: d} \|_{C} \\
=& 2 \sqrt{|\mathcal{R}|} \sum_{d=1}^{D}\left\|\mathbf{p}_{: d}\right\|_{C}\left\|\mathbf{r}_{: d}\right\|_{C}\left\|\mathbf{q}_{: d}\right\|_{C}
\end{aligned}.
\end{equation*}
We can have the equality holds if and only if $\left\|\mathbf{p}_{: d}\right\|_{3}^{3}\left\|\mathbf{r}_{: d}\right\|_{3}^{3}=|\mathcal{R}|\left\|\mathbf{q}_{: d}\right\|_{3}^{3},$ i.e., $\left\|\mathbf{p}_{: d}\right\|_{C}\left\|\mathbf{r}_{: d}\right\|_{C}=\sqrt{|\mathcal{R}|}\left\|\mathbf{q}_{: d}\right\|_{C}$.

For all CP decomposition $\hat{\mathcal{X}}=\sum_{d=1}^{D} \mathbf{p}_{: d} \otimes \mathbf{r}_{: d} \otimes \mathbf{q}_{: d}$, we can always let $\bm{p}'_{:d}=\bm{p}_{:d}$, $\bm{r}'_{:d}=\sqrt{\frac{\left\|\mathbf{q}_{d}\right\|_{C} \sqrt{|\mathcal{R}|}}{\left\|\mathbf{p}_{: d}\right\|_{C}\left\|\mathbf{r}_{: d}\right\|_{C}} }\mathbf{r}_{: d}$ and
$\bm{q}'_{:d}=\sqrt{\frac{\left\|\mathbf{p}_{: d}\right\|_{C}\left\|\mathbf{r}_{: d}\right\|_{C}}{\left\|\mathbf{q}_{\cdot d}\right\|_{C} \sqrt{|\mathcal{R}|}}} \mathbf{p}_{: d}$, then we have
\begin{equation*}
\left\|\mathbf{p}_{: d}^{\prime}\right\|_{C}\left\|\mathbf{r}_{: d}^{\prime}\right\|_{C}=\sqrt{|\mathcal{R}|}\left\|\mathbf{q}_{: d}^{\prime}\right\|_{C}
,
\end{equation*}
and at the same time we  make sure that that $\hat{\mathcal{X}}=\sum_{d=1}^{D} \mathbf{p}_{: d}^{\prime} \otimes \mathbf{r}_{: d}^{\prime} \otimes \mathbf{q}_{: d}^{\prime}$. Therefore, we can have
\begin{equation*}
\begin{aligned}
&\frac{1}{\sqrt{|\mathcal{R}|}} \sum_{j=1}^{|\mathcal{R}|}\left\|\hat{\mathcal{X}}_{j}\right\|_{C} \\=&\frac{1}{2 \sqrt{|\mathcal{R}|}} \sum_{j=1}^{|\mathcal{R}|} \min _{\hat{\mathcal{X}}_{j}=\boldsymbol{x}_i\mathbf{R}_{j}\boldsymbol{x}_k^{\top}}\left(\left\|\boldsymbol{x}_i \mathbf{R}_{j}\right\|_{3}^{3}+\|\boldsymbol{x}_k\|_{3}^{3}\right) \\
 \leq & \frac{1}{2 \sqrt{|\mathcal{R}|}} \min _{\hat{\mathcal{X}}_{j}=\boldsymbol{x}_i\mathbf{ R}_{j} \boldsymbol{x}_k^{\top}} \sum_{j=1}^{|\mathcal{R}|}\left(\left\|\boldsymbol{x}_i \mathbf{R}_{j}\right\|_{3}^{3}+\|\boldsymbol{x}_k\|_{3}^{3}\right) \\
=&\min_{\hat{\mathcal{X}}=\sum_{d=1}^{D} \mathbf{p}_{j} \otimes \mathbf{r}_{i d} \otimes \mathbf{q}_{: d}} \sum_{d=1}^{D}\left\|\mathbf{p}_{: d}\right\|_{C}\left\|\mathbf{r}_{: d}\right\|_{C}\left\|\mathbf{q}_{: d}\right\|_{C} \\
=&\|\hat{\mathcal{X}}\|_{3}.
\end{aligned}
\end{equation*}
And in the same manner, we can have that
\begin{equation*}
\begin{aligned}
\small{
\min}\frac{1}{\sqrt{|\mathcal{R}|}}\sum_{(\boldsymbol{x}_i,\boldsymbol{R}_j,\boldsymbol{x}_k)\in\mathcal{S}}||\boldsymbol{x}_i||_3^3+||\boldsymbol{x}_k||_3^3+\|\boldsymbol{x}_i\mathbf{R}_j-\boldsymbol{x}_k\mathbf{R}_j||_3^3
\\=\|\hat{\mathcal{X}}\|_{3}.
\end{aligned}
\end{equation*}
We can have that the equality holds if and only if $\left\|\mathbf{q}_{: d}\right\|_{C}\left\|\mathbf{r}_{: d}\right\|_{C}=\sqrt{|\mathcal{R}|}\left\|\mathbf{p}_{: d}\right\|_{C}$.
  Then we can see that the conclusion holds if and only if $\left\|\mathbf{p}_{: d}\right\|_{C}\left\|\mathbf{r}_{: d}\right\|_{C}=\sqrt{|\mathcal{R}|}\left\|\mathbf{q}_{: d}\right\|_{C}$ and
$||\bm{q}_{:d}||_C||\bm{r}_{:,d}||_C=\sqrt{|\mathcal{R}|}\left\|\mathbf{p}_{: d}\right\|_{C}, \forall d \in\{1,2, \ldots, D\}$.
Then the proof of Theorem \ref{Theorem3} completes.
\end{proof}

\section{Proof for Theorem \ref{Theorem4}}\label{Proof4t}
\textbf{Theorem 4} Suppose that $\hat{\bm{X}}=\bm{x}_i\overline{\bm{R}}_j\bm{x}^{\top}_k$ for $j=1,2,\cdots,|\mathcal{R}|$, where $\bm{x}_i,\bm{x}_k,\bm{R}_j$ are real matrices and
$\bm{R}_j$ is diagonal. Then, the following equation holds
\begin{equation}\label{bound4}
\begin{aligned}\small{
\frac{1}{4 \sqrt{|\mathcal{R}|}}}\min\sum_{(\boldsymbol{x}_i,\boldsymbol{R}_j,\boldsymbol{x}_k)\in\mathcal{S}}||\boldsymbol{x}_i||_3^3+||\boldsymbol{x}_k||_3^3+\|\boldsymbol{x}_i\mathbf{R}_j+\boldsymbol{x}_k\mathbf{R}_j||_3^3
\\=\|\hat{\mathcal{X}}\|_{3}
\end{aligned}
\end{equation}


First we prove $\|\boldsymbol{x}_i\mathbf{R}_j\|^2|\boldsymbol{x}_k\mathbf{R}_j|+|\boldsymbol{x}_i\mathbf{R}_j|||\boldsymbol{x}_k\mathbf{R}_j||^2\leq
\|\boldsymbol{x}_i\mathbf{R}_j\|_3^3+\|\boldsymbol{x}_k\mathbf{R}_j\|_3^3$ holds.
Notice that $|\boldsymbol{x}_i\mathbf{R}_j|+ |\boldsymbol{x}_k\mathbf{R}_j| \geq 0$ and $(|\boldsymbol{x}_i\mathbf{R}_j|  -|\boldsymbol{x}_k\mathbf{R}_j|)^2\geq 0$. Then  we have
\begin{equation}\label{for inequality a}
\begin{aligned}
(|\boldsymbol{x}_i\mathbf{R}_j|+ |\boldsymbol{x}_k\mathbf{R}_j| )(|\boldsymbol{x}_i\mathbf{R}_j|  -|\boldsymbol{x}_k\mathbf{R}_j|)^2\geq &0\\
(|\boldsymbol{x}_i\mathbf{R}_j|+ |\boldsymbol{x}_k\mathbf{R}_j| )(|\boldsymbol{x}_i\mathbf{R}_j|- |\boldsymbol{x}_k\mathbf{R}_j| )(|\boldsymbol{x}_i\mathbf{R}_j|  -|\boldsymbol{x}_k\mathbf{R}_j|)\geq &0\\
 (|\boldsymbol{x}_i\mathbf{R}_j|^2- |\boldsymbol{x}_k\mathbf{R}_j|^2)(|\boldsymbol{x}_i\mathbf{R}_j |-|\boldsymbol{x}_k\mathbf{R}_j|)\geq&0\\
\end{aligned}
\end{equation}
Then we can derive the following formula:
\begin{equation}
\begin{aligned}
|\boldsymbol{x}_i\mathbf{R}_j|^2(|\boldsymbol{x}_i\mathbf{R}_j|- |\boldsymbol{x}_k\mathbf{R}_j| )  \geq& |\boldsymbol{x}_k\mathbf{R}_j|^2(|\boldsymbol{x}_i\mathbf{R}_j|-|\boldsymbol{x}_k\mathbf{R}_j|)\\
|\boldsymbol{x}_i\mathbf{R}_j|^3-|\boldsymbol{x}_i\mathbf{R}_j|^2|\boldsymbol{x}_k\mathbf{R}_j|\geq& |\boldsymbol{x}_k\mathbf{R}_j|^2|\boldsymbol{x}_i\mathbf{R}_j|-|\boldsymbol{x}_i\mathbf{R}_j|^3\\
\|\small{\boldsymbol{x}_i}\mathbf{R}_j\|_3^3+\|\boldsymbol{x}_k\mathbf{R}_j\|_3^3\geq&\|\small{\boldsymbol{x}_i}\mathbf{R}_j\|^2|\boldsymbol{x}_k\mathbf{R}^{\top}_j|+|\boldsymbol{x}_i\mathbf{R}_j|||\boldsymbol{x}_k\mathbf{R}_j||^2
\end{aligned}
\end{equation}

Then we have:
\begin{equation*}
\begin{aligned}
& \small{\min}\frac{1}{4\sqrt{|\mathcal{R}|}} \sum_{(\boldsymbol{x}_i,\boldsymbol{R}_j,\boldsymbol{x}_k)\in\mathcal{S}}||\boldsymbol{x}_i||_3^3+||\boldsymbol{x}_k||_3^3+\|\boldsymbol{x}_i\mathbf{R}_j+\boldsymbol{x}_k\mathbf{R}_j||_3^3\\
\stackrel{a}{\leq} & \min _{(\boldsymbol{x}_i,\boldsymbol{R}_j,\boldsymbol{x}_k)\in\mathcal{S}} \frac{1}{4\sqrt{|\mathcal{R}|}} \sum_{j=1}^{|\mathcal{R}|}\|\boldsymbol{x}_i\|_{3}^{3}+\|\boldsymbol{x}_k\|_{3}^{3}+\left\|\boldsymbol{x}_i \boldsymbol{R}_{j}\right\|_{3}^{3}\\&+\|\boldsymbol{x}_k \boldsymbol{R}_{j}\|_{3}^{3}+ 3\|\boldsymbol{x}_i\mathbf{R}_j\|^2|\boldsymbol{x}_k\mathbf{R}_j|+3|\boldsymbol{x}_i\mathbf{R}_j|||\boldsymbol{x}_k\mathbf{R}_j||^2 \\
\stackrel{b}{\leq} & \min _{(\boldsymbol{x}_i,\boldsymbol{R}_j,\boldsymbol{x}_k)\in\mathcal{S}} \frac{1}{4\sqrt{|\mathcal{R}|}} \sum_{j=1}^{|\mathcal{R}|}\|\boldsymbol{x}_i\|_{3}^{3}+\|\boldsymbol{x}_k\|_{3}^{3}+\left\|\boldsymbol{x}_i \boldsymbol{R}_{j}\right\|_{3}^{3}\\&+\|\boldsymbol{x}_k \boldsymbol{R}_{j}\|_{3}^{3}+ 3\left\|\boldsymbol{x}_i \boldsymbol{R}_{j}\right\|_{3}^{3}+3\|\boldsymbol{x}_k \boldsymbol{R}_{j}\|_{3}^{3} \\
 \leq& \min _{(\boldsymbol{x}_i,\boldsymbol{R}_j,\boldsymbol{x}_k)\in\mathcal{S}} \frac{1}{4\sqrt{|\mathcal{R}|}} \sum_{j=1}^{|\mathcal{R}|}\|\boldsymbol{x}_i\|_{3}^{3}+\|\boldsymbol{x}_k\|_{3}^{3}+4\left\|\boldsymbol{x}_i \boldsymbol{R}_{j}\right\|_{3}^{3}\\&+4\|\boldsymbol{x}_k \boldsymbol{R}_{j}\|_{3}^{3}\\
\leq & \min _{(\boldsymbol{x}_i,\boldsymbol{R}_j,\boldsymbol{x}_k)\in\mathcal{S}} \small{\frac{1}{\sqrt{|\mathcal{R}|}}} \sum_{j=1}^{|\mathcal{R}|}\|\boldsymbol{x}_i\|_{3}^{3}+\|\boldsymbol{x}_k\|_{3}^{3}+\left\|\boldsymbol{x}_i \boldsymbol{R}_{j}\right\|_{3}^{3}+\|\boldsymbol{x}_k \boldsymbol{R}_{j}\|_{3}^{3}\\
\end{aligned}
\end{equation*}
  Since  $\boldsymbol{x}_i \boldsymbol{R}_{j}$ and $\boldsymbol{x}_k \boldsymbol{R}_{j}$  are all vectors, we can have $\|\boldsymbol{x}_i \boldsymbol{R}_{j}\boldsymbol{x}_k \boldsymbol{R}_{j}\|=\|\boldsymbol{x}_k \boldsymbol{R}_{j}\boldsymbol{x}_i \boldsymbol{R}_{j}\|$. Then the inequality (a) holds.
The inequality (b) holds due to the Eq.(\ref{for inequality a}). 

Then in the same manner  with Eq.(\ref{center formular1}), we can have that
\begin{equation*}
\begin{aligned}
\small{\min}\frac{1}{4\sqrt{|\mathcal{R}|}}\sum_{(\boldsymbol{x}_i,\boldsymbol{R}_j,\boldsymbol{x}_k)\in\mathcal{S}}||\boldsymbol{x}_i||_3^3+||\boldsymbol{x}_k||_3^3+\|\boldsymbol{x}_i\mathbf{R}_j+\boldsymbol{x}_k\mathbf{R}_j||_3^3
\\=\|\hat{\mathcal{X}}\|_{3}.
\end{aligned}
\end{equation*}
We can have that the equality holds if and only if $\left\|\mathbf{q}_{: d}\right\|_{C}\left\|\mathbf{r}_{: d}\right\|_{C}=\sqrt{|\mathcal{R}|}\left\|\mathbf{p}_{: d}\right\|_{C}$.
  Then we can see that the conclusion holds if and only if $\left\|\mathbf{p}_{: d}\right\|_{C}\left\|\mathbf{r}_{: d}\right\|_{C}=\sqrt{|\mathcal{R}|}\left\|\mathbf{q}_{: d}\right\|_{C}$ and
$||\bm{q}_{:d}||_C||\bm{r}_{:,d}||_C=\sqrt{|\mathcal{R}|}\left\|\mathbf{p}_{: d}\right\|_{C}, \forall d \in\{1,2, \ldots, D\}$.
Then the proof of Theorem \ref{Theorem4} completes.

\section{Experimental Details and Appendix}
We implement our model using PyTorch and test it on a single GPU.
Here Table \ref{data} shows statistics of the datasets used in this paper. 
The hypermeters for CP, ComplEx, RESCAL RotatE models are shown in Table \ref{Hyperparameters found by grid search for CP model}, Table \ref{Hyperparameters found by grid search for ComplEx model}, Table \ref{Hyperparameters found by grid search for RESCAL model} and Table \ref{Hyperparameters found by grid search for RotatE model} respectively. We have counted the running time of each epoch for different models with ER in WN18RR as follows: CP with ER takes 58s, ComplEx with ER takes 84s and RESCAL with ER takes 73s.

\noindent\textbf{Study on semantic-similarity hyperparameter $\epsilon$.} In the experiments above, we provide the semantic-similarity parameter $\epsilon_j$ for each relation $\bm{R}_j$ in ER.
 To characterize the similarity between entities adequately and study the impact of $\epsilon$,  here we also conduct another version of ER where we provide $\epsilon_{ik}$ for  $a_{ik}$ (in Eq.(\ref{two-order Semantic Relevance})), which we denote as $ER*$. From Table \ref{eplison}, we can see $ER$ and  $ER*$ have similar performance. It shows providing $\epsilon_j$ for each relation $\bm{R}_j$ in ER is proper.

\begin{table}
	\centering
	\begin{tabular}{lccc}
  \toprule
Datasets    &WN18RR &FB15K237 &YAGO3-10\\
  \midrule
dimension & 2000 &2000 & 2000\\
batch size&100 &100 &500\\
learning rate&0.1&0.05&0.1\\
  \bottomrule
	\end{tabular}
\caption{Hyperparameters found by grid search for CP mdoel.}\label{Hyperparameters found by grid search for CP model}
\end{table}

\begin{table}
	\centering
	\begin{tabular}{lccc}
  \toprule
Datasets    &WN18RR &FB15K237 &YAGO3-10\\
  \midrule
dimension & 2000 &2000 & 2000\\
batch size&200 &200 &1000\\
learning rate&0.05&0.1&0.05\\
  \bottomrule
	\end{tabular}
\caption{Hyperparameters found by grid search for ComplEx mdoel.}\label{Hyperparameters found by grid search for ComplEx model}
\end{table}

\begin{table}
	\centering
	\begin{tabular}{lccc}
  \toprule
Datasets    &WN18RR &FB15K237 &YAGO3-10\\
  \midrule
dimension & 512 &512 & 512\\
batch size&400 &400 &1000\\
learning rate&0.1&0.1&0.05\\
  \bottomrule
	\end{tabular}
\caption{Hyperparameters found by grid search for RESCAL mdoel.}\label{Hyperparameters found by grid search for RESCAL model}
\end{table}

\begin{table}
	\centering
	\begin{tabular}{lccc}
  \toprule
Datasets   &WN18RR &FB15K237 &YAGO3-10\\
  \midrule
dimension & 400 &400 & 400\\
batch size&100 &100 &500\\
learning rate&0.1&0.05&0.05\\
  \bottomrule
	\end{tabular}
\caption{Hyperparameters found by grid search for RotatE mdoel.}\label{Hyperparameters found by grid search for RotatE model}
\end{table}

\begin{table}
	\centering
	\begin{tabular}{lccc}
  \toprule
Model    &MRR &Hits@1 &Hits@10\\
  \midrule
ComplEx-ER &.374 &.282 &.563\\
ComplEx-ER*&.375 &.282 &.565\\
  \bottomrule
	\end{tabular}
\caption{Evaluation results of $\epsilon$ on FB15K237.}\label{eplison}
\end{table}

\begin{table}[t]
\begin{center}
\setlength{\tabcolsep}{1mm}\label{Table5}
 \begin{tabular}{lccccc}
  \toprule
Dataset & $\#$Entity & $\#$Relation & $\#$Training & $\#$Valid & $\#$Test \\
  \midrule
  WN18RR & 40,943& 11& 86,835&3,034&3,134\\
FB15K237& 14,541& 237& 272,115& 17,535& 20,466\\
YAGO3-10 & 123,182 &37& 1,079,040& 5,000& 5,000\\
  \bottomrule
 \end{tabular}
 \caption{\label{data}Statistics of the datasets used in this paper.
}
 \end{center}
\end{table}

\begin{table}
	\centering
	\begin{tabular}{lccc}
  \toprule
Model    &MRR &Hits@1 &Hits@10\\
  \midrule
ComplEx$_0$ & .355&.263&.542\\
ComplEx$_1$&.374 &.282 &.563\\
ComplEx$_2$&.378&.284&.569\\
  \bottomrule
	\end{tabular}
\caption{Evaluation results on FB15K237.}\label{complex1-4}
\end{table}
\begin{figure}
	\centering
\includegraphics[width=2.6in]{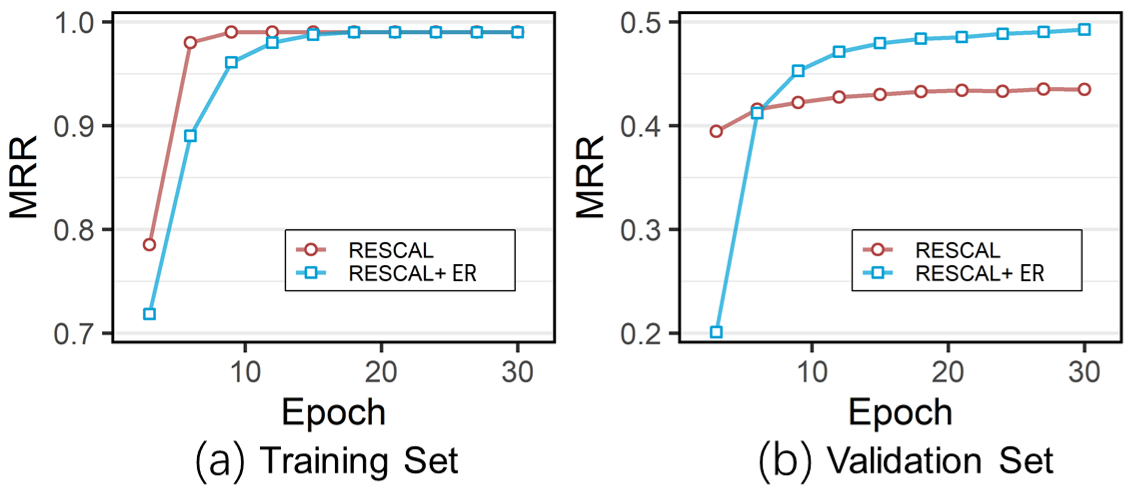}
\caption{Study of training and validation curves.}\label{train_validation}
\end{figure}
\begin{table*}
\begin{center}
 \begin{tabular}{lccccccccc}
  \toprule
  &  \multicolumn{3}{c}{WN18RR}&\multicolumn{3}{c}{FB15K-237}&\multicolumn{3}{c}{YAGO3-10}\\
  Models  &MRR &Hits@1 &Hits@10  &MRR &Hits@1 &Hits@10&MRR &Hits@1 &Hits@10\\
  \midrule
  CP-FRO&.460& - &.480& .340& -& .510& .540& - &.680\\
  CP-N3 &.470 &.430& .544& .354& .261& .544& .577& .505& .705\\
  CP-DURA &.478& .441& .552& .367& .272& .555& .579& .506& .709\\
  CP-ER&\textbf{.482}& \textbf{.444}& \textbf{.557}& \textbf{.371}&\textbf{.275}& \textbf{.561}& \textbf{.584}& \textbf{.508}&\textbf{.712}\\
  \midrule
 ComplEx-FRO& .470& -& .540& .350& -& .530& .573& -& .710\\
 ComplEx-N3 &.489 &.443 &.580 &.366 &.271 &.558 &.577 &.502 &.711\\
 ComplEx-DURA &.491 &.449 &.571 &.371 &.276 &.560 &.584 &.511 &.713\\
 ComplEx-ER& \textbf{.494}& \textbf{.453}& \textbf{.575}& \textbf{.374}& \textbf{.282}& \textbf{.563}& \textbf{.588}& \textbf{.515}&\textbf{.718}\\
   \midrule
   RESCAL-FRO &.397 &.363 &.452 &.323 &.235 &.501 &.474 &.392 &.628\\
   RESCAL-DURA &.498 &.455 &.577 &.368 &.276 &.550 &.579 &.505 &.712\\
   RESCAL-ER& \textbf{.499}& \textbf{.458}& \textbf{.582}& \textbf{.373}& \textbf{.281}& \textbf{.554}&\textbf{.583}& \textbf{.509} &\textbf{.715}\\
  \bottomrule
 \end{tabular}
  \caption{Comparison between DURA, the squared Frobenius norm (FRO), and the nuclear 3-norm
(N3) regularizers (N3 does not apply to RESCAL).  
The
best performance on each model are marked in bold.}
\label{tensor models}
 \end{center}
\end{table*}

\begin{table*}[ht]
\begin{center}
\setlength{\belowcaptionskip}{-0.cm}
 \begin{tabular}{lccccccccc}
  \midrule
  &  \multicolumn{3}{c}{WN18RR}&\multicolumn{3}{c}{FB15K-237}&\multicolumn{3}{c}{YAGO3-10}\\
Models    &MRR &Hits@1 &Hits@10  &MRR &Hits@1 &Hits@10&MRR &Hits@1 &Hits@10\\
    \midrule
 TransE-FRO &.259&.105&.532&.327&.231&.519&.478&.377& .665\\
 TransE-N3& .265& .107&.533&.328&.232&.518&.483&\textbf{.385}&.664\\
 TransE-DURA& .260& .105& .531&.328&.233&.518&.475&.371&.666 \\
 TransE-ER&\textbf{.268}&\textbf{.110}&\textbf{.536}&\textbf{.329}&\textbf{.235}&\textbf{.525}&\textbf{.489}& .384& \textbf{.669}  \\
  \midrule
 RotatE-FRO& .481&.434&.572& .337&.242& .528& .570 &.481& .680\\
 RotatE-N3 &.483& .440& .580&346 &.251 &.538 &.574 &.498  &.701 \\
 RotatE-DURA &.487&.443&.580&.342&.246& .533&.567 &.491  &.702 \\
 RotatE-ER&\textbf{.490}&\textbf{.445}&\textbf{.581}&\textbf{.352}&\textbf{.255}& \textbf{.547}& \textbf{.581} &\textbf{.505}  &\textbf{.704}\\
  \bottomrule
 \end{tabular}
  \caption{Comparison between DURA, the squared Frobenius norm (FRO), and the nuclear 3-norm
(N3) regularizers. The
best performance on each model are marked in bold.}
\label{distance models}
 \end{center}
\end{table*}

\begin{table*}[ht]
\begin{center}
 \begin{tabular}{lccccccccc}
  \toprule
  &  \multicolumn{3}{c}{WN18RR}&\multicolumn{3}{c}{FB15K-237}&\multicolumn{3}{c}{YAGO3-10}\\
  Models  &MRR &Hits@1 &Hits@10  &MRR &Hits@1 &Hits@10&MRR &Hits@1 &Hits@10\\
   \midrule
ComplEx-RHE&.469& .430& .538& .348&.262& .542& .570& .501&.708\\
ComplEx-LLE  &.477& .442& .551& .363& .271& .552& .576& .504& .701\\
ComplEx-TFR &.473&.441& .545& .358& .264&.541&.573& .502& .702\\
ComplEx-EIA  &.463 &.345 &.542 &.356 &.266 &.529 &.573 &.501 &.703\\
ComplEx-Pretrain &.479 &.440&.553 & .353 & .268&.533& .578& .502 &.704\\
ComplEx-ER&.494&.453&.575&.374&.282&.563&.588& .515&.718\\
  \bottomrule
 \end{tabular}
  \caption{Evaluation results of different models on WN18RR, FB15k-237 and YAGO3-10 datasets. 
 }
\label{Comparison of semantic information models}
 \end{center}
\end{table*}

\begin{table*}[th]
\begin{center}
 \begin{tabular}{lccccccccc}
  \toprule
  &  \multicolumn{3}{c}{WN18RR}&\multicolumn{3}{c}{FB15K-237}&\multicolumn{3}{c}{YAGO3-10}\\
  Models  &MRR &Hits@1 &Hits@10  &MRR &Hits@1 &Hits@10&MRR &Hits@1 &Hits@10\\
   \midrule
CP-ER&.479& .441& .556& .371&.273& .560& .582& .506&.709\\
 ComplEx-ER& .492& .452& .574& .371& .275& .560& .586& .514&.712\\
  \bottomrule
 \end{tabular}
  \caption{Evaluation results of ER based on
3-norm on WN18RR, FB15k-237 and YAGO3-10 datasets. 
 }
\label{3-norm MARA with all models}
 \end{center}
\end{table*}

\end{document}